\documentclass[runningheads,a4paper]{llncs}

\usepackage{hyperref}
\usepackage{times}
\usepackage{helvet}
\usepackage{courier}
\usepackage{color}
\usepackage{amsmath}
\usepackage{amssymb}
\usepackage{graphicx}
\usepackage{subcaption}
\usepackage{enumerate}

\usepackage[lined,boxed,commentsnumbered, ruled,noend]{algorithm2e}

\newcommand{\network}{\mathcal{N}}
\newcommand{\rel}{\textbf{Rel}}
\newcommand{\circw}{\diamond}
\newcommand{\uni}{\mathcal{U}}

\newcommand{\sfp}{\textsf{b}}
\newcommand{\sfm}{\textsf{m}}
\newcommand{\sfo}{\textsf{o}}
\newcommand{\sfd}{\textsf{d}}
\newcommand{\sfs}{\textsf{s}}
\newcommand{\sff}{\textsf{f}}
\newcommand{\sfeq}{\textsf{eq}}
\newcommand{\sffi}{\textsf{fi}}
\newcommand{\sfsi}{\textsf{si}}
\newcommand{\sfoi}{\textsf{oi}}
\newcommand{\sfdi}{\textsf{di}}
\newcommand{\sfmi}{\textsf{mi}}
\newcommand{\sfpi}{\textsf{bi}}

\newcommand{\bdc}{{\bf DC}}
\newcommand{\bdr}{{\bf DR}}
\newcommand{\bec}{{\bf EC}}
\newcommand{\beq}{{\bf EQ}}

\newcommand{\bpo}{{\bf PO}}
\newcommand{\bpp}{{\bf PP}}
\newcommand{\bppi}{{\bf PP}^{-1}}
\newcommand{\btpp}{{\bf TPP}}
\newcommand{\btppi}{{\bf TPP}^{-1}}
\newcommand{\bntpp}{{\bf NTPP}}
\newcommand{\bntppi}{{\bf NTPP}^{-1}}

\newcommand{\horn}{\ensuremath{\mathcal{H}}}

\newcommand{\clb}{\ensuremath{\widehat{\mathcal{B}}}}

\newcommand{\Closure}[1]{\ensuremath{\widehat{#1}}}

\newcommand{\qcm}{\mathcal{M}}

\newcommand{\mbr}{\mathcal{M}}

\newcommand{\keywords}[1]{\par\addvspace\baselineskip
\noindent\keywordname\enspace\ignorespaces#1}

\begin{document}

\mainmatter  

\title{On Distributive Subalgebras of Qualitative Spatial and Temporal Calculi\thanks{Work supported by the Australian Research Council  under DP120104159 and FT0990811.}}

\titlerunning{On Distributive Subalgebras of Qualitative Spatial and Temporal Calculi}

%
%
\author{Zhiguo Long
\and
Sanjiang Li%
\thanks{Corresponding author: Sanjiang.Li@uts.edu.au}%
}
%

\institute{Centre for Quantum Computation \& Intelligent Systems,\\ Faculty of Engineering \& Information Technology,\\
University of Technology Sydney, Sydney, Australia
}

%
%

\maketitle

\begin{abstract}
Qualitative calculi play a central role in  representing and reasoning about qualitative spatial and temporal knowledge. This paper studies distributive subalgebras of qualitative calculi, which are subalgebras in which (weak) composition distributives over nonempty intersections. It has been proven for RCC5 and RCC8 that path consistent constraint network over a distributive subalgebra is always minimal and globally consistent (in the sense of strong $n$-consistency) in a qualitative sense. The well-known subclass of convex interval relations provides one such an example of distributive subalgebras. This paper first gives a characterisation of distributive subalgebras, which states that the intersection of a set of $n\geq 3$ relations in the subalgebra is nonempty if and only if the intersection of every two of these relations is nonempty. We further compute and generate all maximal distributive subalgebras for Point Algebra, Interval Algebra, RCC5 and RCC8, Cardinal Relation Algebra, and Rectangle Algebra. Lastly, we  establish two nice properties which will play an important role in efficient reasoning with constraint networks involving a large number of variables.

\keywords{Qualitative Calculi; Qualitative Spatial and Temporal Reasoning; Distributive Subalgebra; Region Connection Calculus; Rectangle Algebra}
\end{abstract}

\section{Introduction}

A dominant part of qualitative spatial and temporal reasoning (QSTR) research focus on the study of individual or multiple qualitative calculi. Roughly speaking, a qualitative calculus $\qcm$ is simply a finite class of relations over a universe $\uni$ of spatial or temporal entities which form a Boolean algebra. Usually, we assume that the identity relation is contained in one atomic relation in $\qcm$ and relations in $\qcm$ are closed under converse \cite{LigozatR04}.
Well-known qualitative calculi include Point Algebra (PA) \cite{VilainK86,Beek89} and Interval Algebra (IA) \cite{Allen83} for representing temporal relations and Region Connection Calculus RCC5 and RCC8 \cite{RandellCC92}, Cardinal Relation Algebra \cite{Frank91,Ligozat98}, and Rectangle Algebra \cite{Guesgen89,BalbianiCC99} for representing spatial relations.

For convenience, we write RCC5/8 for either RCC5 or RCC8. Since the composition of two RCC5/8 relations $R,S$ is not necessarily a relation in RCC5/8 \cite{DuntschWM01,LiY03a}, we write $R\diamond S$ for the smallest relation in RCC5/8 which contains $R\circ S$, the usual composition of $R,S$, and call $R\diamond S$ the \emph{weak composition} of $R,S$ \cite{DuntschWM01,LiY03a}. Unlike RCC5/8, the calculi PA, IA, CRA and RA are closed under composition and are all relation algebras. With weak composition, RCC5/8 is also a relation algebra.  

Using a qualitative calculus $\qcm$, we represent spatial or temporal information in terms of relations in $\qcm$, and formulate a spatial or temporal problem as a set of qualitative constraints (called a \emph{qualitative constraint network} or QCN). A qualitative constraint has the form $(x R y)$, which specifies that two variables $x,y$ are related by the relation $R$ in $\qcm$. A QCN $\network$ is \emph{consistent} if there exists an assignment of values in $\uni$ to variables in $\network$ such that all constraints in $\network$ are satisfied simultaneously. If this is the case, we call this assignment a \emph{solution} of $\network$. We say $\qcm$ is \emph{minimal} if,  for each constraint $(x R y)$ in $\network$, $R$ is the minimal (or \emph{strongest}) relation between $x$ and $y$ that is entailed by $\network$. We say $\network$ is \emph{ globally consistent} if every partial solution (i.e. a partial assignment that satisfies all constraints in a restriction of $\network$) can be extended to a solution of $\network$.   

The \emph{consistency problem} and the \emph{minimal labelling problem} (MLP) are two major reasoning tasks of QSTR research. The consistency problem decides whether a QCN has a solution and the MLP  decides if it is minimal. These problems have been investigated in depth in the past three decades for many qualitative calculi in the literature, see e.g. \cite{Allen83,AmaneddineC12,BalbianiCC99,ChandraP05,AIJPrime,Ligozat98,NebelB95,SioutisK12,Beek89}.

Both problems are in general NP-hard for IA, CRA, RCC5/8, and RA. Local consistency algorithms like path consistency algorithm (PCA) are designed for solving these problems approximately \cite{Allen83}. A QCN $\network=\{v_i R_{ij} v_j: 1\leq i,j\leq n\}$ is \emph{path consistent} (PC) if  
each $R_{ij}$ is non-empty and contained in the (weak) composition of $R_{ik}$ and $R_{kj}$ for any $k$. Applying PCA will either find an inconsistency in $\network$ in case $\network$ is not path consistent, or return a path consistent network that is equivalent to $\network$, which is also known as the \emph{algebraic closure} or \emph{a-closure} of $\network$ \cite{LigozatR04}. A QCN $\network=\{v_i R_{ij} v_j: 1\leq i,j\leq n\}$ is called \emph{basic} or \emph{atomic} if every relation $R_{ij}$ is an atomic relation in $\qcm$.

Interestingly, for every qualitative calculus $\qcm$ mentioned in this paper and any basic QCN $\network$ over $\qcm$, path consistency ensures consistency, i.e. any path consistent network has a solution. This holds also for any tractable subclass $\mathcal{S}$ of $\qcm$ that contains all basic relations and is closed under (weak) composition, intersection, and converse.
For example, it was found in \cite{NebelB95} that there exists a unique maximal tractable subclass of IA, written as $\horn$ and called ORD-Horn. It was proved, also in \cite{NebelB95}, that any path consistent network over $\horn$ is consistent.

In this paper, we study distributive subalgebras of qualitative calculi. We assume $\qcm$ is a qualitative calculus which satisfies the condition that every path consistent basic QCN has a solution. A subalgebra of $\qcm$ is a subclass of $\qcm$ that contains all atomic relations and is closed under (weak) composition, intersection, and converse. A subalgebra $\mathcal{S}$ is \emph{distributive} if (weak) composition distributives over nonempty intersection, i.e. $R\diamond (S\cap T)=(R\diamond S)\cap (R\diamond T)$ and $(S\cap T)\diamond R = (S\diamond R) \cap (T\diamond R)$ for any $R,S,T\in \mathcal{S}$ with $S\cap T\not=\varnothing$. 

Although distributive subalgebra is a new concept proposed recently in \cite{DuckhamLLL14,AIJPrime}, several examples of distributive subalgebras have been studied before. The first such a subalgebra, the subclass of convex IA relations $\mathcal{C}_{\mathrm{IA}}$, was found in \cite{Ligozat94}, where Ligozat also proved that path consistent networks over $\mathcal{C}_{\mathrm{IA}}$ is globally consistent. As every globally consistent network is minimal, this shows that path consistent networks over $\mathcal{C}_{\mathrm{IA}}$ is also minimal. Later, Chandra and Pujari \cite{ChandraP05} defined a class of convex RCC8 relations (written $D_{41}^8$ in \cite{AIJPrime} and this paper) and proved that every path consistent network over $D_{41}^8$ is minimal. More recently, Amaneddine and Condotta \cite{AmaneddineC12} found another subclass of IA, written as $\mathcal{S}_{\mathrm{IA}}$, and proved that $\mathcal{C}_{\mathrm{IA}}$ and $\mathcal{S}_{\mathrm{IA}}$ are the only maximal subalgebras of IA such that path consistent networks over which are globally consistent. It turns out that these subalgebras are all maximal distributive subalgebras of IA or RCC8 \cite{AIJPrime}.

The important concept of distributive subalgebra was also found very useful in identifying a subnetwork that is equivalent to a given one but has no redundant relations. Such a subnetwork is called a \emph{prime} subnetwork in \cite{DuckhamLLL14,AIJPrime}. It was proved there that every constraint network over a distributive subalgebra of RCC5/8 has a unique prime subnetwork, which can be found in cubic time; and, in contrast, it is in general NP-hard to decide if a constraint is non-redundant in an arbitrary RCC5/8 constraint network.  
The cubic time algorithm for finding the prime subnetwork is very useful in  applications such as computing, storing, and compressing the relationships between spatial objects and hence saving space for  storage and communication. We refer the reader to \cite{AIJPrime} for a real-world application example and detailed discussions.


As the focus of \cite{DuckhamLLL14} and \cite{AIJPrime} is redundancy in RCC5/8 constraint networks, there are several interesting topics left untouched, which are the subject of this paper. We first give a characterisation of distributive subalgebras in terms of intersections of relations and then compute and find all maximal distributive subalgebras for every qualitative calculus mentioned before. Lastly, we establish two nice properties regarding partial path consistency \cite{bliek1999path} and variable elimination \cite{ZhangM09} of constraint networks over a distributive subalgebra. These properties will play an important role in efficient reasoning with sparse constraint networks involving a large number of variables. 

The remainder of this paper is organised as follows. In Section 2, we first give a short introduction of the qualitative calculi mentioned above and recall basic notions including weak composition, path and global consistency. Section 3 then presents a characterisation of distributive subalgebras and Section 4 shows how we compute and find all maximal distributive subalgebras of these calculi. We then prove in Section~5 two important properties of distributive subalgebras that will be used in efficient reasoning with large sparse constraint networks. In Section 6 we discuss the connection between distributive subalgebras and conceptual neighbourhood graphs, and relation with classical CSPs. The last section then concludes the paper.

\section{Qualitative Calculi}

In this section, we first recall the qualitative calculi PA, IA, CRA, RCC5/8, and RA, and then,  recall some relevant notions and results of these constraint languages.

Suppose $\uni$ is a domain of spatial or temporal entities. Write $\rel(\uni)$ for the Boolean algebra of binary relations on $\uni$. A \emph{qualitative calculus} \cite{LigozatR04} $\qcm$ on $\uni$ is defined as a finite Boolean subalgebra of $\rel(\uni)$ {which has an atom that contains the identity relation $id_\mathcal{U}$ on $\uni$} and is closed under converse, i.e., $R$ is in $\qcm$ iff its converse $$R^{-1}=\{(a,b)\in \uni\times \uni: (b,a)\in R\}$$ is in $\qcm$ \cite{LigozatR04}. A relation $\alpha$ in a qualitative calculus $\qcm$ is \emph{atomic} or \emph{basic} if it is an atom in $\qcm$. {Note that the set of basic relations of a qualitative calculus is \emph{jointly exhaustive and pairwise disjoint} (JEPD).} Well-known qualitative calculi include, among others, PA \cite{VilainK86,Beek89}, IA \cite{Allen83}, CRA \cite{Frank91,Ligozat98}, RA \cite{Guesgen89,BalbianiCC99}, and RCC5 and RCC8 \cite{RandellCC92}.

\subsection{Point Algebra and Interval Algebra}

\begin{definition}[Point Algebra (PA) \cite{VilainK86}]
Let $\uni$ be the set of real numbers. The Point Algebra is the Boolean subalgebra generated by the JEPD set of relations $\{<,>,=\}$, where $<,>,=$ are defined as usual.
\end{definition}

PA contains eight relations, viz. the three basic relations $<,>,=$, the empty relation, the universal relation $\star$, and three non-basic relations $\leq,\geq,\neq$.

\begin{definition}[Interval Algebra (IA) \cite{Allen83}]
Let $\uni$ be the set of closed intervals on the real line.
Thirteen binary relations between two intervals $x=[x^-,x^+]$ and $y=[y^-,y^+]$ are defined by the order of the four endpoints of $x$ and $y$, see Table~\ref{tab:int}. The Interval Algebra is generated by these JEPD relations.
\end{definition}

We write
\begin{align}
\mathcal{B}_{\mathrm{IA}} = \{\sfp,\sfm,\sfo,\sfs,\sfd,\sff,\sfeq,\sffi,\sfdi,\sfsi,\sfoi,\sfmi,\sfpi\}
\end{align}
for the set of basic IA relations. Ligozat \cite{Ligozat94} defines the {\emph{dimension}} 
of a basic interval relation as 2 minus the number of equalities appearing in the definition of the relation (see Table~\ref{tab:int}). That is, for basic relations we have
\begin{align*}
\dim(\sfeq)=0, \dim(\sfm)=\dim(\sfs)=\dim(\sff)=1, \dim(\sfp)=\dim(\sfo)=\dim(\sfd)=2.
\end{align*}
For a non-basic relation $R$ we define
\begin{align}\label{lemma:dim-1}
\dim(R)=\max\{\dim(\theta): \theta\ \mbox{is a basic relation in}\ R\}.
\end{align}

\begin{table}[htbp]
\centering
\begin{tabular}{ccc}
\begin{tabular}{|c|c|c|c|c|}
  \hline
 Relation & Symb. & Conv. & Dim. &  Definition  \\ \hline
 before & \sfp & \sfpi & 2 & $x^+<y^-$  \\
 meets & \sfm & \sfmi & 1 & $x^+=y^-$  \\
 overlaps & \sfo & \sfoi & 2 & $x^-<y^-<x^+<y^+$ \\
 starts & \sfs & \sfsi & 1 & $x^-=y^-<x^+<y^+$  \\
 during & \sfd & \sfdi & 2 & $y^-<x^-<x^+<y^+$ \\
 finishes & \sff & \sffi & 1 & $y^-<x^-<x^+=y^+$ \\
 equals & \sfeq & \sfeq & 0 & $x^-=y^-<x^+=y^+$  \\
  \hline
\end{tabular}
&
\quad\quad 
&
\begin{tabular}{c}
 \includegraphics[width=.3\textwidth]{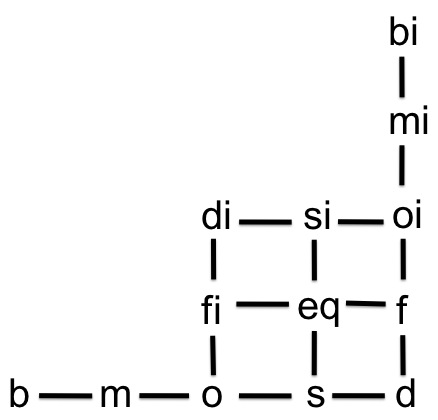}
 \end{tabular}
 \\
 (i) & \quad\quad & (ii)
\end{tabular}
\caption{IA basic relations (i) definitions and (ii) conceptual neighbourhood graph, where $x=[x^-,x^+],y=[y^-,y^+]$ are two intervals.} \label{tab:int}
\end{table}

Using the conceptual neighbourhood graph (CNG) of IA \cite{Freksa92a}, Ligozat \cite{Ligozat94} gives a geometrical characterisation for ORD-Horn relations. Consider the CNG of IA (shown in Table~\ref{tab:int} (ii)) as a {partially ordered set} $(\mathcal{B}_{int}, \preceq)$ (by interpreting any relation to be smaller than its right or upper neighbours). For {$\theta_1,\theta_2\in\mathcal{B}_{int}$ with $\theta_1 \preceq \theta_2$, we write $[\theta_1,\theta_2]$ as the set of basic interval relations $\theta$ such that $\theta_1\preceq \theta \preceq \theta_2$}, and call such a relation a \emph{convex} interval relation. An IA relation $R$ is called {\emph{pre-convex}} if it can be obtained from a convex relation by removing one or more basic relations with dimension lower than $R$. For example, $[\sfo,\sfeq]=\{\sfo,\sfs,\sffi,\sfeq\}$ is a convex relation and $\{\sfo,\sfeq\}$ is a pre-convex relation. Ligozat has shown that ORD-Horn relations are precisely pre-convex relations. Every path consistent  network over $\horn$ is consistent \cite{NebelB95}. In addition, every path consistent  network over $\mathcal{C}_{\mathrm{IA}}$ is globally consistent and minimal \cite{Ligozat94}.

\subsection{RCC5 and RCC8}
{The RCC5/8 constraint language is a fragment of the Region Connection Calculus (RCC) \cite{RandellCC92}. The RCC is a first order theory based on a binary connectedness relation and has canonical models defined over connected topological spaces \cite{stell2000bca,LiY03a}.}
Since applications in GIS and many other spatial reasoning tasks mainly consider objects represented in the real plane, in this paper, we interpret regions as non-empty regular closed sets in the plane, and say two regions are \emph{connected} if they have non-empty intersection.

\begin{definition}[RCC5 and RCC8 Algebras] \label{ex:rcc}
Let $\uni$ be the set of non-empty regular closed sets, or \emph{regions}, in the real plane. The RCC8 algebra is generated by the eight topological relations
\begin{equation*}\label{eq:rcc8}
\bdc,\bec,\bpo,\beq,\btpp,\bntpp,\btppi,\bntppi,
\end{equation*}
where $\bdc,\bec,\bpo,\btpp$ and $\bntpp$ are defined in Table~\ref{tab:RCC8}, $\beq$ is the identity relation, and $\btppi$ and $\bntppi$ are the converses of $\btpp$ and $\bntpp$ respectively (see Fig.~\ref{fig:basicRCC} for illustration). 
RCC5 is the sub-algebra of RCC8 generated by the five part-whole relations
\begin{equation*}\label{eq:rcc5}
\bdr,\bpo,\beq,\bpp,\bppi,
\end{equation*}
where $\bdr=\bdc\cup\bec$, $\bpp=\btpp\cup\bntpp$, and $\bppi=\btppi\cup\bntppi$.
\end{definition}

\begin{table}\centering
\begin{tabular}{c|c|c|c}
 Relation & Definition & Relation & Definition \\ \hline
 \bdc & $a\cap b=\varnothing$  & \btpp & $a\subset b$, $a\not\subset b^\circ$ \\
 \bec & $a\cap b\neq\varnothing$, $a^\circ\cap b^\circ=\varnothing$ & \bntpp & $a\subset b^\circ$\\
 \bpo & $a\nsubseteq b$, $b\nsubseteq a$, $a^\circ\cap b^\circ\neq\varnothing$ & \beq & $a=b$
\end{tabular}
\caption{Topological interpretation of basic RCC8 relations in the plane, where
$a,b$ are regions, and $a^\circ,b^\circ$ are the interiors of $a,b$, respectively.}\label{tab:RCC8}
\end{table}

\begin{figure}
\centering
  \includegraphics[width=.75\textwidth]{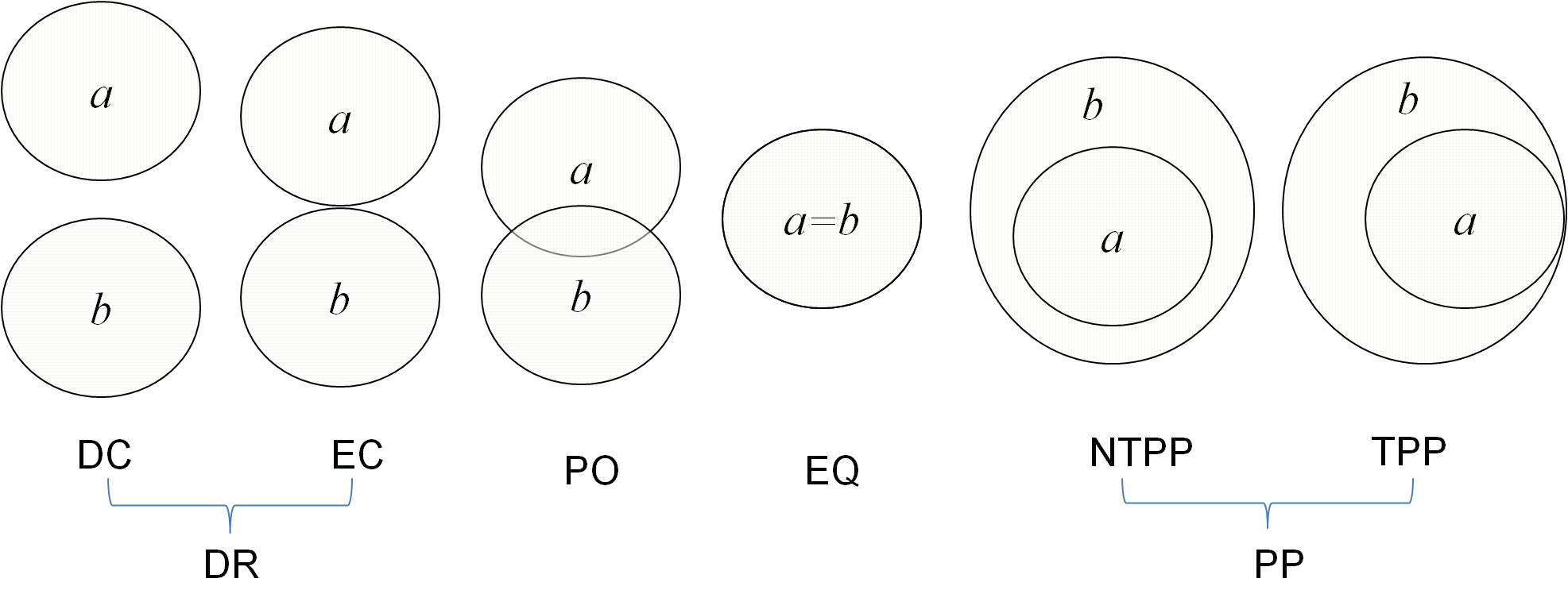}
  \caption{Illustration for basic relations in RCC5 / RCC8}\label{fig:basicRCC}
\end{figure}

\subsection{Cardinal Relation Algebra and Rectangle Algebra}

\begin{definition}[Cardinal Relation Algebra (CRA) \cite{Frank91,Ligozat98}]
Let $\uni$ be the real plane. Define binary relations $NW,N,NE,W,EQ,E,SW,S,SE$ as in Fig.~\ref{tab:CRA}.
The Cardinal Relation Algebra is generated by these nine JEPD relations.

\begin{figure}[htbp]
\begin{minipage}[b]{0.4\textwidth}\centering
\begin{tabular}{c|c}
Relation & Definition \\
\hline
NW & $x<x',y>y'$ \\
N &  $x=x',y>y'$ \\
NW & $x>x',y>y'$ \\
W  & $x<x',y=y'$ \\
EQ & $x=x',y=y'$ \\
E  & $x>x',y=y'$ \\
SW & $x<x',y<y'$ \\
S  & $x=x',y<y'$ \\
SW & $x>x',y<y'$ \\
\end{tabular}
\caption{Basic relations of CRA.}
\label{tab:CRA}
\end{minipage}
\begin{minipage}[b]{0.6\textwidth}\centering
\includegraphics[width=.9\textwidth]{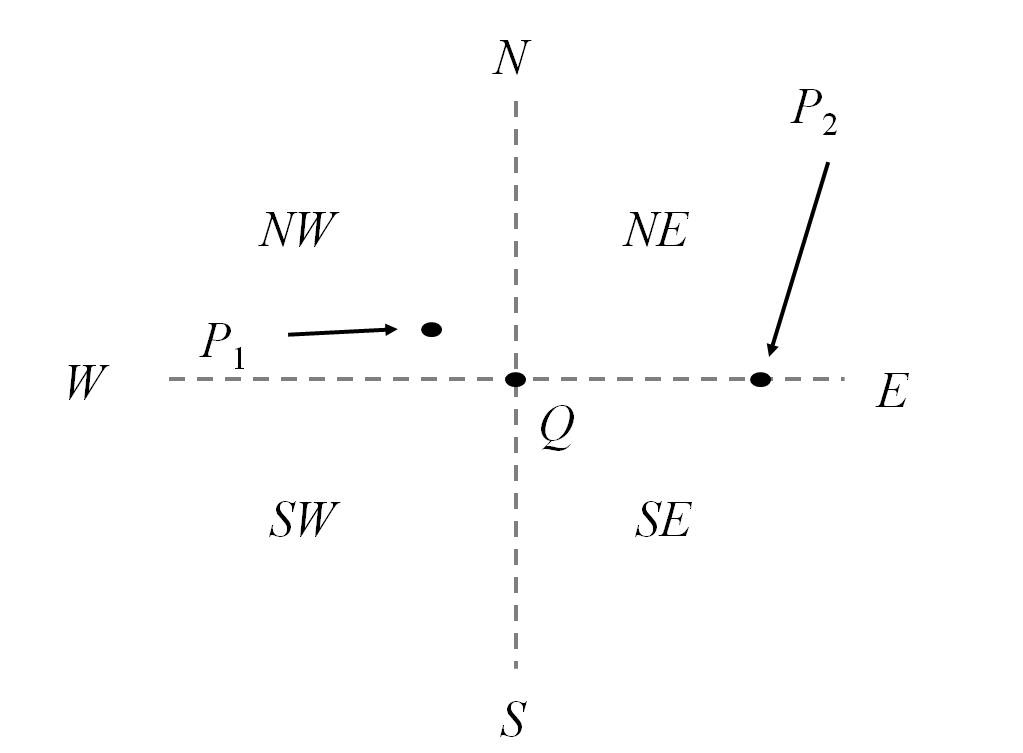}
\caption{Examples: $P_1\ \text{NW}\ Q$ and $P_2\ \text{E}\ Q$}
\label{fig:CRA}
\end{minipage}
\end{figure}

\end{definition}

\begin{figure}[htb]
\centering
\begin{tabular}{cc}

\includegraphics[width=.4\textwidth]{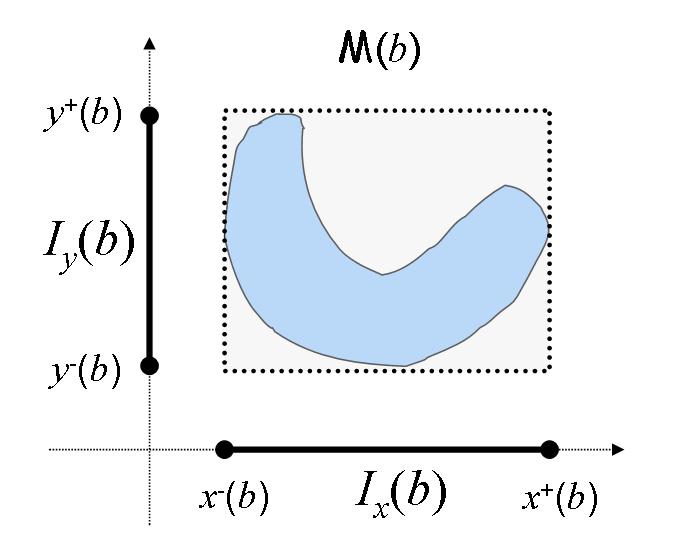}
&
\includegraphics[width=.4\textwidth]{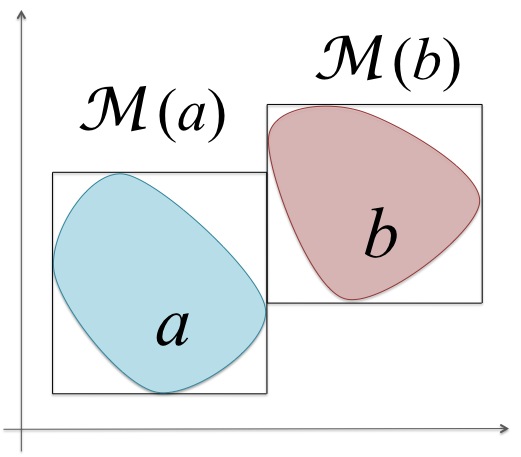}
\\
(i) & (ii)
 \end{tabular}
\caption{(i) The minimum bounding rectangle $\mbr(b)$ of a region $b$; (ii) the RA relation of $a$ to $b$ is $\sfm\otimes\sfo$. \label{mbr}}
\label{fig:mbr}
\end{figure}

CRA can be viewed as an extension  of  PA to the plane.
Similarly, IA can also be extended to regions in the plane. We assume an orthogonal basis in the Euclidean plane. For a bounded region $a$, its \emph{minimum bounding rectangle} (MBR), denoted by $\mbr(a)$, is the smallest rectangle which contains $a$ and whose sides are parallel to the axes of the basis.  We write $I_x(a)$ and $I_y(a)$ as, respectively, the $x$- and $y$-projections of $\mbr(a)$. The basic rectangle relation between two bounded regions $a,b$ is $\alpha\otimes\beta$ iff  $(I_x(a),I_x(b))\in\alpha$ and $(I_y(a),I_y(b))\in\beta$, where $\alpha,\beta$ are two basic IA relations (see Figure~\ref{fig:mbr} for illustration). We write $\mathcal{B}_{\mathrm{RA}}$ for the set of basic rectangle relations, \label{projection}
i.e., \label{otimes}
\begin{align}
\mathcal{B}_{\mathrm{RA}}= \{\alpha\otimes\beta:\alpha,\beta\in\mathcal{B}_{\mathrm{IA}}\}.
\end{align}
There are 169 different basic rectangle relations in $\mathcal{B}_{\mathrm{RA}}$. The Rectangle Algebra (RA) is the algebra generated by relations in $\mathcal{B}_{\mathrm{RA}}$ \cite{BalbianiCC99}.

Henceforth, for two IA relations $R,S$, we will write $R \otimes S$ for the (non-basic) relation $\{\alpha \otimes \beta: \alpha \in R, \beta \in S, \alpha,\beta \in \mathcal{B}\}$; analogously, for two subclasses of IA relations $\mathcal{R}$ and $\mathcal{S}$, we will write $\mathcal{R} \otimes \mathcal{S}$ for the set of RA relations $\{R \otimes S: R \in \mathcal{R}, S \in \mathcal{S}\}$. The following lemma is straightforward.
\begin{lemma}
Let $\Delta=\{v_i {(R_{ij}\otimes S_{ij})} v_j\}_{i,j=1}^n$ be an RA network, where ${R_{ij}}$ and ${S_{ij}}$ are arbitrary IA relations. Then $\Delta$ is satisfiable iff its projections $\Delta^x=\{x_i {R_{ij}}x_j\}_{i,j=1}^n$ and
$\Delta^y=\{y_i {S_{ij}} y_j\}_{i,j=1}^n$ are satisfiable IA networks.
\end{lemma}

{As a consequence, we know $\mathcal{H}\otimes \mathcal{H}$ is a tractable subclass of RA. No maximal tractable subclass has been identified for RA, but a larger tractable subclass of RA has been identified in \cite{BalbianiCC99}.}

\subsection{Properties of Qualitative Calculi}

While PA, IA, CRA and RA are all closed under composition, the composition of two basic RCC5/8 relations is not necessarily a relation in RCC5/8 \cite{DuntschWM01,LiY03a}. 

For two RCC5/8 relations $R$ and $S$, recall that we write $ R \diamond S $ for the weak composition of $ R $ and $ S $. 
Suppose $\alpha,\beta,\gamma$ are three basic RCC5/8 relations. Then we have
\begin{align}\label{eq:wc}
\gamma\in \alpha \circw \beta \Leftrightarrow \gamma \cap {(\alpha \circ \beta)} \not=\varnothing.
\end{align}
The weak composition of two (non-basic) relations $R$ and $S$ is computed as follows: 
\begin{align*}
R\circw S &=\bigcup\{\alpha\circw \beta: {\alpha\in R, \beta\in S}\}.
\end{align*}

Because PA, IA, CRA and RA are closed under composition, we have
\begin{proposition}
For $\qcm$ being PA, IA, CRA or RA, weak composition is the same as composition, i.e. for any $R,S\in\qcm$, we have $R\circ S=R\circw S$.
\end{proposition}

\begin{proposition}[See \cite{duntsch2005relation}]\label{prop:cycle}
With the weak composition operation $\circw${, the converse operation $^{-1}$, and the identity relation,} PA, IA, RCC5/8, CRA, and RA are relation algebras. {In particular, the weak composition operation $\circw$ is associative.}  
Moreover, for PA, IA, RCC5/8, CRA, and RA relations $R,S,T$, we have the following cycle law
\begin{align}\label{eq:cycle}
(R\circw S)\cap T\not=\varnothing  \Leftrightarrow (R^{-1} \circw T)\cap S\not=\varnothing &\Leftrightarrow (T\circw S^{-1})\cap R\not=\varnothing.
\end{align}
\end{proposition}
Figure~\ref{fig:cycle} gives an illustration of the cycle law. 

\begin{figure}[h]
\centering
\begin{tabular}{c}
\includegraphics[width=.6\columnwidth]{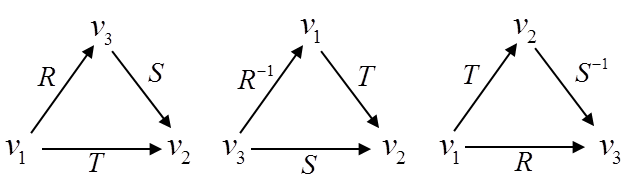} \\
\end{tabular}
\caption{Illustration of the cycle law (from \cite{AIJPrime}).}\label{fig:cycle}
\end{figure}

In the following, we assume $\circw$ takes precedence over $\cap$.


We say a network $\network=\{v_i R_{ij}v_j: { 1\leq i,j\leq n} \}$ is \emph{path consistent} if for every $1\leq i,j,k\leq n$, we have
\begin{align*}
{\varnothing\not=}R_{ij} \subseteq R_{ik}\diamond R_{kj}.
\end{align*}
In general, path consistency can be enforced by calling the following rule until an empty constraint occurs (then $\network$ is inconsistent) or the network becomes stable 
\begin{align*}
R_{ij} \leftarrow (R_{ik}\diamond R_{kj}) \cap R_{ij},
\end{align*}
where $1\leq i,j,k\leq n$ are arbitrary. A cubic time algorithm, henceforth called the \emph{path consistency algorithm} or PCA, has been devised to enforce path consistency. For any {qualitative constraint network} $\network$, the PCA either detects inconsistency of $\network$ or returns a path consistent network, written $\network_p$, which is equivalent to $\network$ and also known as the \emph{algebraic closure} or \emph{a-closure} of $\network$ \cite{LigozatR04}. It is easy to see that in this case $\network_p$ refines $\network$, i.e., we have $S_{ij} \subseteq R_{ij}$ for each constraint $(v_i S_{ij} v_j)$ in $\network_p$.

\begin{definition} \label{dfn:minimal-network}
Let $\qcm$ be a qualitative calculus with universe $\uni$. Suppose $\network=\{v_i T_{ij} v_j: 1\leq i,j\leq n\}$ is a QCN over $\qcm$ and $V=\{v_1,...,v_n\}$. For a pair of variables $v_i,v_j \in V$ ($i\not=j$) and a basic relation $\alpha$ in $T_{ij}$, we say $\alpha$ is \emph{feasible} if there exists a solution $(a_1,a_2,\ldots, a_n)$ in $U$ of $\network$ such that $(a_i,a_j)$ is an instance of $\alpha$. We say $\network$ is \emph{minimal} if $\alpha$ is feasible for every pair of variables $v_i,v_j$ ($i\not=j$) and every basic relation $\alpha$ in $T_{ij}$.

A \emph{scenario} of $\network$ is a basic network with form $\Theta=\{v_i \theta_{ij} v_j: 1\leq i,j\leq n\}$, where each $\theta_{ij}$ is a basic relation in $T_{ij}$. A scenario is consistent if it has a solution. We say $\network$ is \emph{weakly globally consistent} (\emph{globally consistent}, respectively) if any consistent scenario (solution, respectively) of $\network{\downarrow}_{V'}$ can be extended to a consistent scenario (solution, respectively) of $\network$, where $V'$ is any nonempty subset of $V$ and $\network{\downarrow}_{V'}$ is the restriction of $\network$ to $V'$.
\end{definition}

It is clear that every (weakly) globally consistent network is consistent and minimal.

In the following, we assume the qualitative calculus $\qcm$ has the following properties:
\begin{align}\label{eq:ra}
\mbox{$\qcm$ is a relation algebra with operations $\circw$, $id_\mathcal{U}$, and $^{-1}$;} \\
\label{eq:pc->c}
\mbox{Every path consistent basic network over $\qcm$ is consistent.}
\end{align}

\section{Distributive Subalgebras}

\begin{definition} \cite{AIJPrime} \label{dfn:distributive}
{Let $\qcm$ be a qualitative calculus. A subclass $\mathcal{S}$ of $\qcm$ is called} a \emph{subalgebra}  if $\mathcal{S}$ contains all basic relations and is closed under converse, weak composition, and intersection. A subalgebra $\mathcal{S}$ is \emph{distributive} if weak composition distributes over non-empty intersections of relations in $\mathcal{S}$, i.e. $R\diamond (S\cap T)=(R\diamond S)\cap (R\diamond T)$ and $(S\cap T)\diamond R = (S\diamond R) \cap (T\diamond R)$ for any $R,S,T\in \mathcal{S}$ with $S\cap T\not=\varnothing$. 
\end{definition}

Suppose $\mathcal{X}$ is a subclass of $\qcm$. We write $\Closure{\mathcal{X}}$ for the subalgebra of $\qcm$ generated by $\mathcal{X}$, i.e. $\Closure{\mathcal{X}}$ is the closure of $\mathcal{X}$ in $\qcm$ under intersection, weak composition, and converse. In particular, $\widehat{\mathcal{B}}$ denotes the closure of $\mathcal{B}$ in $\qcm$.  

\begin{proposition}
Let $\qcm$ be one of the calculi PA, IA, RCC5/8, CRA, RA and $\mathcal{B}$ the set of basic relations of $\qcm$. Then $\widehat{\mathcal{B}}$ is a distributive subalgebra.
\end{proposition}

This shows that the above definition of distributive subalgebra is well-defined for these calculi and every distributive subalgebra of $\mathcal{M}$ contains $\widehat{\mathcal{B}}$ as a subclass.

\subsection{Distributive Subalgebra Is Helly}
Helly's theorem \cite{danzer1963helly} is a very useful result in discrete geometry.  For $n$ convex subsets of $\mathbb{R}$, it says if the intersection of any two of them is non-empty, then the intersection of the whole collection is also non-empty. Interestingly, relations in a distributive subalgebra have a similar property as convex sets in the real line and, moreover, relations having such property are exactly those in a distributive subalgebra.

\begin{definition}
A subclass $\mathcal{S}$ of a qualitative calculus is called \emph{Helly} if, for every $R,S,T\in \mathcal{S}$, we have 
\begin{equation} \label{eq:helly}
R\cap S\cap T\not=\varnothing \quad \mbox{iff} \quad   R\cap S\not=\varnothing,\ R\cap T\not=\varnothing, \ S\cap T\not=\varnothing.
\end{equation}
\end{definition}
If $\mathcal{S}$ is a subalgebra, then it is straightforward to prove that $\mathcal{S}$ is Helly if and only if, for any $n$ relations $R_1,...,R_n$ in $\mathcal{S}$, we have \begin{equation} \label{eq:helly+}
\bigcap_{i=1}^n R_i \not=\varnothing \quad \mbox{iff} \quad (\forall 1\leq i\not=j\leq n)\  R_i\cap R_j\not=\varnothing
\end{equation}

The following result is first proved  for RCC5/8 in \cite{AIJPrime}. Following a similar proof, it is straightforward to show this holds in general.
\begin{lemma}[\cite{AIJPrime}]\label{lem:3x}
Suppose $\qcm$ is a qualitative calculus that satisfies \eqref{eq:ra}, i.e. $\qcm$, with the weak composition, the converse operation, and the identity relation, is a relation algebra. Then every distributive subalgebra of $\qcm$ is Helly.
\end{lemma}
Surprisingly, the above condition is also sufficient.

\begin{theorem}\label{thm:dis=helly}
Suppose $\qcm$ is a qualitative calculus that satisfies \eqref{eq:ra}.  Let $\mathcal{S}$ be a subalgebra of $\qcm$. Then $\mathcal{S}$ is distributive if and only if it is Helly.
\end{theorem}

\begin{proof}
Since Lemma~\ref{lem:3x} already shows the ``only if'' part, we only need to show the ``if'' part. Suppose $R,S,T$ are three relations in $\mathcal{S}$. We first note  
$R\circw (S\cap T)\subseteq R\circw S\cap R\circw T$. Furthermore, for any basic relation $\gamma$, by using the cycle law twice, we have
\begin{align*}
\gamma\not\in R\circw(S\cap T) 
&\Leftrightarrow \{\gamma\}\cap R\circw(S\cap T)=\varnothing \\
& 
\Leftrightarrow R^{-1}\circw \gamma \cap S\cap T =\varnothing \\
& \Leftrightarrow R^{-1}\circw \gamma \cap S=\varnothing \ \mbox{or}\ R^{-1}\circw\gamma\cap T=\varnothing\\
&\Leftrightarrow \{\gamma\}\cap R\circw S=\varnothing \ \mbox{or}\ \{\gamma\}\cap R\circw T=\varnothing\\
&\Leftrightarrow \gamma\not\in R\circw S\ \mbox{or}\ \gamma\not\in R\circw T.
\end{align*}
This shows $R\circw (S\cap T)=R\circw S\cap R\circw T$.
That is, $\mathcal{S}$ is Helly only if it is distributive.
\qed
\end{proof}

\subsection{Path Consistency Implies Weakly Global Consistency}

We have the following very important result for distributive subalgebras.
\begin{theorem}\label{thm:pc-wgc}
Let $\qcm$ be a qualitative calculus that satisfies \eqref{eq:ra} and \eqref{eq:pc->c}. Suppose $\mathcal{S}$ is a distributive subalgebra of $\qcm$. Then every path consistent network over $\mathcal{S}$ is weakly globally consistent and minimal.
\end{theorem}
This result was first proved for RCC5/8 in \cite{AIJPrime}. If every path consistent network over $\mathcal{S}$ is \emph{consistent}, then, following the proof in \cite[Theorem~18]{AIJPrime}, we can show every path consistent network over $\mathcal{S}$ is also weakly globally consistent and minimal. From the analysis in the following section, we can easily see that this is the case for PA, IA, CRA, RA, and RCC5/8. To show the general case, the proof is given in Appendix.

\section{Maximal Distributive Subalgebras}

A distributive subalgebra $\mathcal{S}$ is \emph{maximal} if there is no other distributive subalgebra that properly contains $\mathcal{S}$. In this section, we compute and list all maximal distributive subalgebras for RA, IA, CRA, RA, and RCC5/8.

\subsection{Maximal Distributive Subalgebras of PA, IA, RCC5, and RCC8}

Let $\mathcal{X}$ be a subclass of $\qcm$. Recall we write $\Closure{\mathcal{X}}$ for the subalgebra of $\qcm$ generated by $\mathcal{X}$ and write $\mathcal{B}$ for the set of basic relations in $\qcm$.
For $\qcm$ being PA, IA, RCC5, or RCC8, to compute the maximal distributive subalgebras of $\qcm$, we first compute $\clb$, and then check by a program if $\Closure{\clb\cup \mathcal{Z}}$ satisfies distributivity for some subset $\mathcal{Z}$ of $\mathcal{M}$.

Write $\mathcal{D}$ for the set of relations $R$ in $\mathcal{M}$ such that $\Closure{\clb\cup\{R\}}$ satisfies distributivity. We then check for every pair of relations $R,S$ in $\mathcal{D}$ if $\Closure{\clb\cup\{R,S\}}$ satisfies distributivity. If this is the case, then we say  $R$ has d-relation to $S$. Fortunately, the result shows that there are precisely two disjoint subsets $\mathcal{X}$ and $\mathcal{Y}$ (which form a partition of $\mathcal{D}$) such that each relation $R$ in $\mathcal{X}$ ($\mathcal{Y}$, respectively) has d-relation to every other relation in $\mathcal{X}$ ($\mathcal{Y}$, respectively), but has no d-relation to any relation in $\mathcal{Y}$ ($\mathcal{X}$, respectively).  Moreover, $\Closure{\clb\cup \mathcal{X}}$ and $\Closure{\clb\cup \mathcal{Y}}$ are both distributive subalgebras of $\mathcal{M}$. It is clear that these are the only maximal distributive subalgebras of $\mathcal{M}$. 

In the following, we list the maximal distributive subalgebras of PA and IA and refer to  \cite[Appendix B]{AIJPrime} for those of RCC5 and RCC8.

\subsubsection{PA.}
The closure of basic relations of PA contains 4 non-empty relations 
\begin{align}
\clb_{\mathrm{PA}}=\{<,>,=,\star\}.
\end{align}
One of the maximal distributive subalgebras contains 6 non-empty relations
\begin{align}
<,>,=,\star,\le,\ge,
\end{align}
which is exactly the subclass $\mathcal{C}_{\mathrm{PA}}$ of convex PA relations; the other  contains 5 non-empty relations 
\begin{align}
<,>,=,\star,\ne,
\end{align}
which is exactly the subclass $\mathcal{S}_{\mathrm{PA}}$ identified in \cite{AmaneddineC12}.

\subsubsection{IA.}
The closure of basic IA relations, $\clb_{\mathrm{IA}}$, contains 29 non-empty relations (see Table~\ref{tab:IAC}). Our computation shows that IA has two maximal distributive subalgebra, one contains additional 53 non-empty relations, shown in Table~\ref{tab:IAC1}, which is exactly the subclass $\mathcal{C}_{\mathrm{IA}}$ of convex IA relations; the other contains additional 52 non-empty relations, shown in Table~\ref{tab:IAC2}, which is exactly the subclass $\mathcal{S}_{\mathrm{IA}}$ identified in \cite{AmaneddineC12}.

\begin{table}
  \centering
  \caption{The closure of basic IA relations, $\mathcal{B}_{\mathrm{IA}}$, contains 29 non-empty relations.}\label{tab:IAC}
    \begin{tabular}{lcr}
     \sfeq &\quad &   \sfd{}  \sfoi{} \sff{}  \\
     \sffi{} &\quad &  \sfd{}  \sfo{}  \sfs{}  \\
     \sff{}  &\quad &  \sfd{}  \sfdi{} \sfo{}  \sfoi{} \sfs{}  \sfsi{} \sff{}  \sffi{} \sfeq{} \\
     \sff{}  \sffi{} \sfeq{} &\quad &  \sfpi{} \\
     \sfsi{} &\quad &  \sfpi{} \sfoi{} \sfmi{} \\
     \sfs{}  &\quad &  \sfpi{} \sfdi{} \sfoi{} \sfmi{} \sfsi{} \\
     \sfs{}  \sfsi{} \sfeq{} &\quad &  \sfpi{} \sfd{}  \sfoi{} \sfmi{} \sff{}  \\
     \sfmi{} &\quad &  \sfpi{} \sfd{}  \sfdi{} \sfo{}  \sfoi{} \sfmi{} \sfs{}  \sfsi{} \sff{}  \sffi{} \sfeq{} \\
     \sfm{}  &\quad &  \sfp{}  \\
     \sfoi{} &\quad &  \sfp{}  \sfo{}  \sfm{}  \\
     \sfo{}  &\quad &  \sfp{}  \sfdi{} \sfo{}  \sfm{}  \sffi{} \\
     \sfdi{} &\quad &  \sfp{}  \sfd{}  \sfo{}  \sfm{}  \sfs{}  \\
     \sfdi{} \sfoi{} \sfsi{} &\quad &  \sfp{}  \sfd{}  \sfdi{} \sfo{}  \sfoi{} \sfm{}  \sfs{}  \sfsi{} \sff{}  \sffi{} \sfeq{} \\
     \sfdi{} \sfo{}  \sffi{} &\quad &  \sfp{}  \sfpi{} \sfd{}  \sfdi{} \sfo{}  \sfoi{} \sfm{}  \sfmi{} \sfs{}  \sfsi{} \sff{}  \sffi{} \sfeq{} \\
     \sfd{}  &\quad &  \\
    \end{tabular}%
\end{table}%
\begin{table}
  \centering
  \caption{Additional relations contained in $\mathcal{C}_{\mathrm{IA}}$.}\label{tab:IAC1}
    \begin{tabular}{lcccr}
    \sffi{} \sfeq{} &\quad &  \sfdi{} \sffi{} &\quad &  \sfd{}  \sfdi{} \sfo{}  \sfoi{} \sfmi{} \sfs{}  \sfsi{} \sff{}  \sffi{} \sfeq{} \\
     \sff{}  \sfeq{} &\quad &  \sfdi{} \sfsi{} &\quad &  \sfd{}  \sfdi{} \sfo{}  \sfoi{} \sfm{}  \sfs{}  \sfsi{} \sff{}  \sffi{} \sfeq{} \\
     \sfsi{} \sfeq{} &\quad &  \sfdi{} \sfsi{} \sffi{} \sfeq{} &\quad &  \sfd{}  \sfdi{} \sfo{}  \sfoi{} \sfm{}  \sfmi{} \sfs{}  \sfsi{} \sff{}  \sffi{} \sfeq{} \\
     \sfs{}  \sfeq{} &\quad &  \sfdi{} \sfoi{} \sfsi{} \sff{}  \sffi{} \sfeq{} &\quad &  \sfpi{} \sfmi{} \\
     \sfoi{} \sff{}  &\quad &  \sfdi{} \sfoi{} \sfmi{} \sfsi{} &\quad &  \sfpi{} \sfoi{} \sfmi{} \sff{}  \\
     \sfoi{} \sfsi{} &\quad &  \sfdi{} \sfoi{} \sfmi{} \sfsi{} \sff{}  \sffi{} \sfeq{} &\quad &  \sfpi{} \sfoi{} \sfmi{} \sfsi{} \\
     \sfoi{} \sfsi{} \sff{}  \sfeq{} &\quad &  \sfdi{} \sfo{}  \sfs{}  \sfsi{} \sffi{} \sfeq{} &\quad &  \sfpi{} \sfoi{} \sfmi{} \sfsi{} \sff{}  \sfeq{} \\
     \sfoi{} \sfmi{} &\quad &  \sfdi{} \sfo{}  \sfm{}  \sffi{} &\quad &  \sfpi{} \sfdi{} \sfoi{} \sfmi{} \sfsi{} \sff{}  \sffi{} \sfeq{} \\
     \sfoi{} \sfmi{} \sff{}  &\quad &  \sfdi{} \sfo{}  \sfm{}  \sfs{}  \sfsi{} \sffi{} \sfeq{} &\quad &  \sfpi{} \sfd{}  \sfoi{} \sfmi{} \sfs{}  \sfsi{} \sff{}  \sfeq{} \\
     \sfoi{} \sfmi{} \sfsi{} &\quad &  \sfd{}  \sff{}  &\quad &  \sfpi{} \sfd{}  \sfdi{} \sfo{}  \sfoi{} \sfm{}  \sfmi{} \sfs{}  \sfsi{} \sff{}  \sffi{} \sfeq{} \\
     \sfoi{} \sfmi{} \sfsi{} \sff{}  \sfeq{} &\quad &  \sfd{}  \sfs{}  &\quad &  \sfp{}  \sfm{}  \\
     \sfo{}  \sffi{} &\quad &  \sfd{}  \sfs{}  \sff{}  \sfeq{} &\quad &  \sfp{}  \sfo{}  \sfm{}  \sffi{} \\
     \sfo{}  \sfs{}  &\quad &  \sfd{}  \sfoi{} \sfs{}  \sfsi{} \sff{}  \sfeq{} &\quad &  \sfp{}  \sfo{}  \sfm{}  \sfs{}  \\
     \sfo{}  \sfs{}  \sffi{} \sfeq{} &\quad &  \sfd{}  \sfoi{} \sfmi{} \sff{}  &\quad &  \sfp{}  \sfo{}  \sfm{}  \sfs{}  \sffi{} \sfeq{} \\
     \sfo{}  \sfm{}  &\quad &  \sfd{}  \sfoi{} \sfmi{} \sfs{}  \sfsi{} \sff{}  \sfeq{} &\quad &  \sfp{}  \sfdi{} \sfo{}  \sfm{}  \sfs{}  \sfsi{} \sffi{} \sfeq{} \\
     \sfo{}  \sfm{}  \sffi{} &\quad &  \sfd{}  \sfo{}  \sfs{}  \sff{}  \sffi{} \sfeq{} &\quad &  \sfp{}  \sfd{}  \sfo{}  \sfm{}  \sfs{}  \sff{}  \sffi{} \sfeq{} \\
     \sfo{}  \sfm{}  \sfs{}  &\quad &  \sfd{}  \sfo{}  \sfm{}  \sfs{}  &\quad &  \sfp{}  \sfd{}  \sfdi{} \sfo{}  \sfoi{} \sfm{}  \sfmi{} \sfs{}  \sfsi{} \sff{}  \sffi{} \sfeq{} \\
     \sfo{}  \sfm{}  \sfs{}  \sffi{} \sfeq{} &\quad &  \sfd{}  \sfo{}  \sfm{}  \sfs{}  \sff{}  \sffi{} \sfeq{} &\quad &  \\
    \end{tabular}%
\end{table}%
\begin{table}
  \centering
  \caption{Additional relations contained in $\mathcal{S}_{\mathrm{IA}}$.}\label{tab:IAC2}
    \begin{tabular}{lcccr}
     \sff{}  \sffi{} &\quad &  \sfpi{} \sfd{}  \sfdi{} \sfo{}  \sfoi{} \sfs{}  \sfsi{} &\quad &  \sfp{}  \sfd{}  \sfdi{} \sfo{}  \sfoi{} \sfm{}  \sfs{}  \sfsi{} \\
     \sfs{}  \sfsi{} &\quad &  \sfpi{} \sfd{}  \sfdi{} \sfo{}  \sfoi{} \sfs{}  \sfsi{} \sff{}  \sffi{} \sfeq{} &\quad &  \sfp{}  \sfpi{} \sfd{}  \sfdi{} \sfo{}  \sfoi{} \\
     \sfdi{} \sfoi{} &\quad &  \sfpi{} \sfd{}  \sfdi{} \sfo{}  \sfoi{} \sfmi{} &\quad &  \sfp{}  \sfpi{} \sfd{}  \sfdi{} \sfo{}  \sfoi{} \sff{}  \sffi{} \\
     \sfdi{} \sfo{}  &\quad &  \sfpi{} \sfd{}  \sfdi{} \sfo{}  \sfoi{} \sfmi{} \sff{}  \sffi{} &\quad &  \sfp{}  \sfpi{} \sfd{}  \sfdi{} \sfo{}  \sfoi{} \sfs{}  \sfsi{} \\
     \sfd{}  \sfoi{} &\quad &  \sfpi{} \sfd{}  \sfdi{} \sfo{}  \sfoi{} \sfmi{} \sfs{}  \sfsi{} &\quad &  \sfp{}  \sfpi{} \sfd{}  \sfdi{} \sfo{}  \sfoi{} \sfs{}  \sfsi{} \sff{}  \sffi{} \sfeq{} \\
     \sfd{}  \sfo{}  &\quad &  \sfp{}  \sfo{}  &\quad &  \sfp{}  \sfpi{} \sfd{}  \sfdi{} \sfo{}  \sfoi{} \sfmi{} \\
     \sfd{}  \sfdi{} \sfo{}  \sfoi{} &\quad &  \sfp{}  \sfdi{} \sfo{}  &\quad &  \sfp{}  \sfpi{} \sfd{}  \sfdi{} \sfo{}  \sfoi{} \sfmi{} \sff{}  \sffi{} \\
     \sfd{}  \sfdi{} \sfo{}  \sfoi{} \sff{}  \sffi{} &\quad &  \sfp{}  \sfdi{} \sfo{}  \sffi{} &\quad &  \sfp{}  \sfpi{} \sfd{}  \sfdi{} \sfo{}  \sfoi{} \sfmi{} \sfs{}  \sfsi{} \\
     \sfd{}  \sfdi{} \sfo{}  \sfoi{} \sfs{}  \sfsi{} &\quad &  \sfp{}  \sfdi{} \sfo{}  \sfm{}  &\quad &  \sfp{}  \sfpi{} \sfd{}  \sfdi{} \sfo{}  \sfoi{} \sfmi{} \sfs{}  \sfsi{} \sff{}  \sffi{} \sfeq{} \\
     \sfpi{} \sfoi{} &\quad &  \sfp{}  \sfd{}  \sfo{}  &\quad &  \sfp{}  \sfpi{} \sfd{}  \sfdi{} \sfo{}  \sfoi{} \sfm{}  \\
     \sfpi{} \sfdi{} \sfoi{} &\quad &  \sfp{}  \sfd{}  \sfo{}  \sfs{}  &\quad &  \sfp{}  \sfpi{} \sfd{}  \sfdi{} \sfo{}  \sfoi{} \sfm{}  \sff{}  \sffi{} \\
     \sfpi{} \sfdi{} \sfoi{} \sfsi{} &\quad &  \sfp{}  \sfd{}  \sfo{}  \sfm{}  &\quad &  \sfp{}  \sfpi{} \sfd{}  \sfdi{} \sfo{}  \sfoi{} \sfm{}  \sfs{}  \sfsi{} \\
     \sfpi{} \sfdi{} \sfoi{} \sfmi{} &\quad &  \sfp{}  \sfd{}  \sfdi{} \sfo{}  \sfoi{} &\quad &  \sfp{}  \sfpi{} \sfd{}  \sfdi{} \sfo{}  \sfoi{} \sfm{}  \sfs{}  \sfsi{} \sff{}  \sffi{} \sfeq{} \\
     \sfpi{} \sfd{}  \sfoi{} &\quad &  \sfp{}  \sfd{}  \sfdi{} \sfo{}  \sfoi{} \sff{}  \sffi{} &\quad &  \sfp{}  \sfpi{} \sfd{}  \sfdi{} \sfo{}  \sfoi{} \sfm{}  \sfmi{} \\
     \sfpi{} \sfd{}  \sfoi{} \sff{}  &\quad &  \sfp{}  \sfd{}  \sfdi{} \sfo{}  \sfoi{} \sfs{}  \sfsi{} &\quad &  \sfp{}  \sfpi{} \sfd{}  \sfdi{} \sfo{}  \sfoi{} \sfm{}  \sfmi{} \sff{}  \sffi{} \\
     \sfpi{} \sfd{}  \sfoi{} \sfmi{} &\quad &  \sfp{}  \sfd{}  \sfdi{} \sfo{}  \sfoi{} \sfs{}  \sfsi{} \sff{}  \sffi{} \sfeq{} &\quad &  \sfp{}  \sfpi{} \sfd{}  \sfdi{} \sfo{}  \sfoi{} \sfm{}  \sfmi{} \sfs{}  \sfsi{} \\
     \sfpi{} \sfd{}  \sfdi{} \sfo{}  \sfoi{} &\quad &  \sfp{}  \sfd{}  \sfdi{} \sfo{}  \sfoi{} \sfm{}  &\quad &  \\
     \sfpi{} \sfd{}  \sfdi{} \sfo{}  \sfoi{} \sff{}  \sffi{} &\quad &  \sfp{}  \sfd{}  \sfdi{} \sfo{}  \sfoi{} \sfm{}  \sff{}  \sffi{} &\quad &  \\
    \end{tabular}%
\end{table}%

\subsection{Maximal Distributive Subalgebras of CRA}
The procedure to compute the maximal distributive subalgebras of CRA is similar to the procedure for PA, IA, RCC5 and RCC8, but with some differences.

First, we compute $\clb$, and then check by a program if $\Closure{\clb\cup \mathcal{Z}}$ satisfies distributivity for some subset $\mathcal{Z}$ of CRA.

Write $\mathcal{D}$ for the set of relations $R$ in CRA such that $\Closure{\clb\cup\{R\}}$ satisfies distributivity. There are 8 different subalgebras in the set of subalgebras $\{\Closure{\clb\cup\{R\}} : R \in \mathcal{D}\}$. We call these 8 distributive subalgebras the \emph{seed} subalgebras. Among these, only 4 are not contained in any other ones. We call these the \emph{candidate} subalgebras. We then verify the following three facts:
\begin{enumerate}
 \item For any pair of different candidate subalgebras $\mathcal{S}_i$ and $\mathcal{S}_j$, we have $\Closure{\mathcal{S}_i \cup \mathcal{S}_j}$ is not distributive.
 \item For any pair of non-candidate subalgebras $\mathcal{S}_i$ and $\mathcal{S}_j$, we have $\Closure{\mathcal{S}_i \cup \mathcal{S}_j}$ is either a candidate subalgebra or not distributive.
 \item For any pair of subalgebras $\mathcal{S}_i$ and $\mathcal{S}_j$ s.t. $\mathcal{S}_i$ is a candidate subalgebra, $\mathcal{S}_j$ is a non-candidate subalgebra, and $\mathcal{S}_j \not\subseteq \mathcal{S}_i$, we have $\Closure{\mathcal{S}_i \cup \mathcal{S}_j}$ is not distributive.
\end{enumerate}
Based upon the above facts, we show that the four candidate subalgebras are the only maximal distributive subalgebras of CRA. 

To prove the maximality, suppose $\mathcal{S}$ is one of the four candidate subalgebras. Let $R$ be a relation in CRA which is not in $\mathcal{S}$. Then $\Closure{\mathcal{S} \cup \{R\}}$ is not distributive. This is because, by the above facts either $\Closure{\clb \cup \{R\}}$ is not distributive or $\Closure{\clb \cup \{R\}}$ is one of the 8 subalgebras and $\Closure{\clb \cup \{R\} \cup \mathcal{S}}$ is not distributive.

{To prove there are no other maximal distributive subalgebras, suppose $\mathcal{S}'$ is a distributive subalgebra that is not a subset of any of the four candidate subalgebras. $\mathcal{S}'$ must contain at least two relations in $\mathcal{D}$, say $R_1$ and $R_2$. By the above facts, we know the closure of the union of $\Closure{\clb \cup \{R_1\}}$ and $\Closure{\clb \cup \{R_2\}}$ is either not distributive or one of the four maximal distributive subalgebra. If it is the latter case, then $\mathcal{S}'$ would be either not distributive or a superset of one of the four maximal distributive subalgebras. Note that the latter situation cannot happen as it contradicts the maximality of the four maximal distributive subalgebras.}

Interestingly, these four maximal distributive subalgebras of CRA correspond exactly to the Cartesian products of the maximal distributive subalgebras of PA, viz. $\mathcal{C}_{\mathrm{PA}}\otimes\mathcal{C}_{\mathrm{PA}},\mathcal{C}_{\mathrm{PA}}\otimes\mathcal{S}_{\mathrm{PA}},\mathcal{S}_{\mathrm{PA}}\otimes\mathcal{C}_{\mathrm{PA}},\mathcal{S}_{\mathrm{PA}}\otimes\mathcal{S}_{\mathrm{PA}}$, where we interpret in a natural way a CRA relation e.g. $\{NW,N\}$ as $\{<,=\}\otimes\{>\}$.

\subsection{Maximal Distributive Subalgebras of RA}
Unlike the other small calculi we have discussed, RA has a large number (169) of basic relations, resulting a total of $2^{169}$ relations in it. It becomes infeasible to exploit the former brute-force procedure to compute the maximal distributive subalgebras of RA. However, noting that the maximal distributive subalgebras of CRA are exactly the Cartesian products of the two maximal distributive subalgebras of PA, we conjecture that a similar situation happens to RA. This is indeed true.

\begin{theorem}
RA has exactly four maximal distributive subalgebras, which are the Cartesian products of the two maximal distributive subalgebras of IA.
\end{theorem}
\begin{proof}
For convenience, we write $\mathcal{D}_1$ and $\mathcal{D}_2$ for the maximal distributive subalgebras $\mathcal{C}_{\mathrm{IA}}$ and $\mathcal{S}_{\mathrm{IA}}$. 
It is straightforward to show that their Cartesian products $\mathcal{D}_i\otimes \mathcal{D}_j$ ($1\leq i,j\leq 2$) are all distributive subalgebras of RA. 

In order to show the maximality of $\mathcal{D}_i\otimes \mathcal{D}_j$, suppose $R \not\in \mathcal{D}_i \otimes \mathcal{D}_j$. We show that the subalgebra $\widehat{\{R\} \cup \mathcal{D}_i \otimes \mathcal{D}_j}$ is not distributive. Let $R_x = \{\alpha \in \mathcal{B}_{\mathrm{IA}} \;|\; \exists \beta \in \mathcal{B}_{\mathrm{IA}} \mbox{ s.t. } (\alpha,\beta) \in R\}$ and define $R_y$ similarly. Note that  $R$ is always contained in $R_x \otimes R_y$. There are two cases. 

Case 1. $R \subsetneq R_x \otimes R_y$. Then there exist $\alpha_0 \in R_x$ and $\beta_0 \in R_y$ s.t. $\alpha_0\otimes\beta_0 \not\in R$. Let $S = \{\alpha_0\}\otimes\star$ and $T = \star\otimes \{\beta_0\}$. Note that $\clb_{\mathrm{RA}}$ is strictly contained in $\mathcal{D}_i \otimes \mathcal{D}_j$. Thus $S,T \in \Closure{\{R\} \cup \mathcal{D}_i \otimes \mathcal{D}_j}$. It is easy to see that $R \cap S \ne \varnothing$, $R \cap T \ne \varnothing$, and $S \cap T \ne \varnothing$, but $R \cap S \cap T = \varnothing$. By Theorem~\ref{thm:dis=helly}, this implies that $\Closure{\{R\} \cup \mathcal{D}_i \otimes \mathcal{D}_j}$ is not distributive.

Case 2. $R = R_x \otimes R_y$. Then we have either $R_x \not\in \mathcal{D}_i$ or $R_y \not\in \mathcal{D}_j$. Take $R_x \not\in \mathcal{D}_i$ as an example. Then $\Closure{\{R_x\} \cup \mathcal{D}_i}$ is not distributive. This implies that there exist $R_0, S_0, T_0 \in \Closure{\{R_x\} \cup \mathcal{D}_i}$ which do not satisfy Helly's condition \eqref{eq:helly}. Note that $R_0\otimes\star$, $S_0\otimes\star$, and $T_0\otimes\star$ are all in $\Closure{\{R\} \cup \mathcal{D}_i \otimes \mathcal{D}_j}$. 
However, the three relations $R_0\otimes\star$, $S_0\otimes\star$, and $T_0\otimes\star$ do not satisfy \eqref{eq:helly}, which means that $\Closure{\{R\} \cup \mathcal{D}_i \otimes \mathcal{D}_j}$ is not distributive.

The above proves the maximality of $\mathcal{D}_i \otimes \mathcal{D}_j$. To show the uniqueness, suppose $\mathcal{S}$ is a distributive subalgebra. We show $\mathcal{S}$ is a subset of $\mathcal{D}_i \otimes \mathcal{D}_j$ for some $i,j$.

First, we show for every $R \in \mathcal{S}$ we have $R = R_x \otimes R_y$. Suppose not. Then there exist $\alpha \in R_x$ and $\beta \in R_y$ s.t. $\alpha\otimes \beta \not\in R$. Similar to the proof of the maximality, we know both $\{\alpha\}\otimes\star$ and  $\star\otimes\{\beta\}$ are in $\Closure{\mathcal{B}}$ and, hence, in $\mathcal{S}$. The three relations $R, \{\alpha\}\otimes\star,\star\otimes\{\beta\}$, however, do not satisfy Helly's condition \eqref{eq:helly}.

Next, we show that $\mathcal{S}$ is a subset of $\mathcal{D}_i \otimes \mathcal{D}_j$ for some $i,j$. Write $\mathcal{S}_x=\{R_x: R\in\mathcal{S}\}$ and $\mathcal{S}_y=\{R_y:R\in\mathcal{S}\}$. We assert that $\mathcal{S}_x$ and $\mathcal{S}_y$ are both distributive subalgebras of IA. We first note that if $R=R_x\otimes R_y\in\mathcal{S}$, then both $R_x\otimes\star$ and $\star\otimes R_y$ are in $\mathcal{S}$. This is because, for instance, $\{\sfeq\}\otimes\star$ is a relation in $\clb_{\mathrm{RA}}\subseteq \mathcal{S}$ and $(R_x\otimes R_y)\circw(\{\sfeq\}\otimes\star)=R_x\otimes\star$. It is easy to check that $\{R_x\otimes\star: R_x\otimes R_y\in \mathcal{S}\}$ is a distributive subalgebra which is contained in $\mathcal{S}$. Now, it is clear that $\mathcal{S}_x$ is a distributive subalgebra of IA and, hence, contained in either $\mathcal{D}_1$ or $\mathcal{D}_2$. The same conclusion applies to $\mathcal{S}_y$. Therefore, $\mathcal{S}$ is a subset of $\mathcal{D}_i \otimes \mathcal{D}_j$ for some $i,j$. \qed
\end{proof}

The above proof also applies to CRA.

\section{Partial Path Consistency and Variable Elimination}
In this section, we present two nice properties of distributive subalgebras, which will play an important role in reasoning with large sparse constraint networks. 

\subsection{Variable Elimination}

In \cite{ZhangM09}, Zhang and Marisetti proposed a novel variable elimination method for solving (classical and finite) connected row convex (CRC) constraints \cite{deville}. The idea is to eliminate the variables one by one until a trivial problem is reached. Although very simple, the algorithm is able to make use of the sparsity of the problem instances and performs very well. One key property of CRC constraints is that any strong path consistent CRC constraint network is globally consistent. Recall that a similar property has been identified  in our  Theorem~\ref{thm:pc-wgc} for constraint networks over a distributive subalgebra. The following theorem shows that the same variable elimination method also applies to constraint networks over a distributive subalgebra, 

\begin{theorem}\label{thm:ve}
Let $\qcm$ be a qualitative calculus that satisfies \eqref{eq:ra} and \eqref{eq:pc->c}. Suppose $\network=\{v_i R_{ij} v_j\;|\; 1\leq i,j\leq n\}$ is a network over a distributive subalgebra $\mathcal{S}$ of $\qcm$ and $V=\{v_1,...,v_n\}$. If $R_{ij}\subseteq R_{in}\circw R_{nj}$ for every $1\leq i,j< n$, then $\network_{{-n}}$ is consistent only if $\network$ is consistent, where $\network_{{-n}}$ is the restriction of $\network$ to $\{v_1,...,v_{n-1}\}$.
\end{theorem}
\begin{proof}
Suppose $\{\delta_{ij}:1\leq i,j<n\}$ is a consistent scenario of $\network_{-n}$. First, write $T_i$ for $R_{n,i}$ and let $\widehat{T}_i=\bigcap_{j=1}^{n-1} T_j\circw\delta_{ji}$. We only need to show $\widehat{T}_j\subseteq \widehat{T}_i\circw\delta_{ij}$. Note
\begin{align*}
\widehat{T}_i\circw\delta_{ij} = (\bigcap_{j=1}^{n-1} T_{j'}\circw \delta_{j'i})\circw \delta_{ij} = \bigcap_{j'=1}^{n-1} (T_{j'}\circw \delta_{j'i}\circw \delta_{ji}) \supseteq \bigcap_{j'=1}^{n-1} T_{j'}\circw\delta_{j'j} =\widehat{T}_j.
\end{align*}

Second, we show $\widehat{T}_i$ is not empty. To this end, by Helly's condition \eqref{eq:helly}, we only need to show 
$T_j\circw \delta_{ji}\cap T_{j'}\circw \delta_{j'i}\not=\varnothing$ for any $j\not=j'$. Using the cycle law twice, we have 
\begin{align*}
T_j\circw \delta_{ji}\cap T_{j'}\circw \delta_{j'i} \not=\varnothing \quad \mbox{iff} \quad & T_{j'}\circw \delta_{j'i}\circw \delta_{ij}\cap T_j\not=\varnothing\\
\quad \mbox{iff} \quad  & {T_{j'}}^{-1}\circw T_j \cap \delta_{j'i}\circw \delta_{ij}\not=\varnothing\\
\quad \mbox{iff} \quad & R_{j'n}\circw R_{nj}\cap \delta_{j'i}\circw\delta_{ij}\not=\varnothing.
\end{align*}
Because $\delta_{j'j}\subseteq R_{j'n}\circw R_{nj}$ and $\delta_{j'j}\subseteq \delta_{j'i}\circw \delta_{ij}$, we have $R_{j'n}\circw R_{nj}\cap \delta_{j'i}\circw\delta_{ij}\not=\varnothing$, hence  
$T_j\circw \delta_{ji}\cap T_{j'}\circw \delta_{j'i}\not=\varnothing$.
\qed
\end{proof}

By the previous theorem, we can directly devise an efficient variable elimination algorithm for constraint networks over a distributive subalgebra. At each step, we choose the node for deleting which has the smallest degree. In particular, we can simply remove all nodes with degree 1 from the constraint network without affecting its consistency. This is especially useful for efficient reasoning with large sparse constraint networks.  

\subsection{Partial Path Consistency}

Another efficient method for solving sparse constraint networks is the partial path consistency (PPC) algorithm proposed by Bliek and Sam-Haroud \cite{bliek1999path}. The idea is to enforce path consistency (PC) on sparse graphs by triangulating instead of completing them. The authors demonstrated that, as far as CRC constraints are concerned, the pruning capacity of PC on triangulated graphs and their completion are identical on the common edges. Recently, PPC has also been extended to qualitative spatial and temporal constraint solving \cite{ChmeissC11,SioutisK12}, where the authors proved that any PPC constraint network over a maximal tractable subclass of IA or RCC8 is always consistent. However, for constraint networks over these subclasses, the pruning capacity of PC on triangulated graphs and their completion may be not identical on the common edges. In this section, we show that the answer is affirmative for constraint networks over distributive subalgebras.

We first recall several basic notions related to PPC introduced in \cite{bliek1999path}.

An undirected graph $G=(V,E)$ is \emph{triangulated} or \emph{chordal} if every cycle of length greater than 3 has a chord, i.e. an edge connecting two non-consecutive vertices of the cycle. For each $v\in V$, the adjacency set $Adj(v)$, is defined as $\{w\in V: \{v,w\}\in E\}$. A vertex $v$ is \emph{simplicial} if $Adj(v)$ is complete. Every chordal graph has a simplicial vertex. Moreover, after removing a simplicial vertex and its incident edges from the graph, a chordal graph remains chordal. The order in which simplicial vertices are successively removed is called a \emph{perfect elimination order}. 

\begin{lemma}[\cite{bliek1999path}]
	If $G=(V,E)$ is an incomplete chordal graph, then one can add a missing edge $(u,w)$ with $u,w\in V$ such that 
	\begin{itemize}
		\item the graph $G' = (V,E\cup \{\{u,w\}\})$ is chordal graph; and
		\item the graph induced by $X=\{x|\{u,x\},\{x,w\} \in E\}$ is complete. 
	\end{itemize}
\end{lemma}

For a constraint network $\network=\{v_i R_{ij} v_j: 1\leq i,j\leq n\}$ over $V=\{v_1,...,v_n\}$, the \emph{constraint graph} of $\network$ is the undirected graph $G(\network)=(V,E(\network))$, for which we have $\{v_i,v_j\}\in E(\network)$ iff $R_{ij}\not=\star$. Given a constraint network $\network$ and a graph $G=(V,E)$, we say $\network$ is \emph{partial path consistent w.r.t. $G$} iff for any $1\leq i,j,k\leq n$ with $\{v_i,v_j\}, \{v_j,v_k\}, \{v_i,v_k\}\in E$ we have $R_{ik}\subseteq R_{ij}\circw R_{jk}$ \cite{ChmeissC11}. 

The following result was first proved for RCC8 in \cite{SioutisCL15}. The proof given there is also applicable to other calculi. We here give a slightly different proof which does not use the weakly global consistency result. 

\begin{theorem}
Let $\qcm$ be a qualitative calculus that satisfies \eqref{eq:ra} and \eqref{eq:pc->c}. Suppose $\network=\{v_i R_{ij} v_j\;|\; 1\leq i,j\leq n\}$ is a network over a distributive subalgebra $\mathcal{S}$ of $\qcm$ and $V=\{v_1,...,v_n\}$. Assume in addition that $G=(V,E)$ is a chordal graph such that $E(\network)\subseteq E$. Then enforcing partial path consistency on $G$ is equivalent to enforcing path consistency on the completion of $G$, in the sense that the relations computed for the constraints in $G$ are identical.
\end{theorem}
\begin{proof}
The proof is similar to the one given for CRC constraints  \cite[Theorem~3]{bliek1999path}. Suppose we have a chordal graph $G = (V,E)$ such that $G(\network)\subseteq G$ and $\network$ is PPC w.r.t. $G$. We will add to $G$ the missing edges one by one until the graph is complete. To prove the theorem, we show that the relations of the constraints can be computed from the existing ones so that each intermediate graph, including the complete graph, is path consistent.
    
    In the following we assume the order $v_1,\ldots,v_n$ is a perfect elimination order of chordal graph $G$. Denote $S_i = \{v_{n-i+1}, \ldots, v_n\}$, $G_i=G(S_i)$ (the induced subgraph of $G$ by $S_i$), and $F_i = \{v_k \in N(v_{n-i}) : k>n-i\}$, where $N(v_{n-i}) = \{v_j : \{v_j,v_{n-i}\} \in E\}$.

	We add the missing edges one by one to $G$ in the following manner:
	\begin{enumerate}
		\item choose the largest $i$ such that $G_i$ is complete; 
		\item choose vertices $v_{n-i},v_j$ in $G$;
		\item label the edge $\{v_{n-i},v_j\}$ (and resp. its reverse) with 
		$$R_{n-i,j} = \bigcap_{v_k \in F_i} R_{n-i,k} \circw R_{k,j}.$$
	\end{enumerate}
	After adding one edge, we prove $G'$, the resulting graph, is still path consistent.
	
	First, we show the added label is non-empty. To show this, by Theorem~\ref{thm:dis=helly}, we need only show $R_{n-i,k} \circw R_{k,j} \cap R_{n-i,k'} \circw R_{k',j} \ne \varnothing$ for any $v_k \ne v_{k'} \in F_i$. Such a pairwise intersection is not empty because, by the cycle law of relation algebra, we have 
    $$R_{n-i,k} \circw R_{k,j} \cap R_{n-i,k'} \circw R_{k',j} \ne \varnothing\ \quad \mbox{iff} \quad R_{k,n-i} \circw R_{n-i,k'} \cap R_{k,j} \circw R_{j,k'} \ne \varnothing.$$ 
   Since $G(F_i \cup \{v_{n-i}\})$ and $G_i$ are complete and path consistent, we have $R_{k,k'} \subseteq R_{k,n-i} \circw R_{n-i,k'}$ and $R_{k,k'} \subseteq R_{k,j} \circw R_{j,k'}$. This shows $R_{k,n-i} \circw R_{n-i,k'} \cap R_{k,j} \circw R_{j,k'} \ne \varnothing$ and, hence, $R_{n-i,k} \circw R_{k,j} \cap R_{n-i,k'} \circw R_{k',j} \ne \varnothing$.

	We then need to show the constraint network is path consistent for the three paths  $\langle n-i,j,k'\rangle$, $\langle n-i,k',j\rangle$, and $\langle k',n-i,j\rangle$. 
	
	For $\langle n-i,j,k'\rangle$, note that, for any $k \in F_i$, we have $R_{n-i,k'} \subseteq R_{n-i,k} \circw R_{k,k'} \subseteq R_{n-i,k} \circw R_{k,j} \circw R_{j,k'}$.  Therefore, we have $R_{n-i,k'} \subseteq \bigcap_{k \in F_i} R_{n-i,k} \circw R_{k,j} \circw R_{j,k'}$. By distributivity, we know $R_{n-i,k'} \subseteq (\bigcap_{k\in F_i} R_{n-i,k} \circw R_{k,j} ) \circw R_{j,k'}=R_{n-i,j} \circw R_{k,j}$. 
	
	For $\langle n-i,k',j\rangle$, by the construction of $R_{n-i,j}$, we have $R_{n-i,j} \subseteq R_{n-i,k'} \circw R_{k',j}$.
	
	For $\langle k',n-i,j\rangle$, we need to show $R_{k',j} \subseteq R_{k',n-i} \circw R_{n-i,j}$. Note $R_{n-i,j} = \bigcap_{v_k \in F_i} R_{n-i,k} \circw R_{k,j}$. By distributivity, it is sufficient to show, for each $k\in F_i$, $R_{k',j} \subseteq R_{k',n-i} \circw R_{n-i,k} \circw R_{k,j}$. Because $G(F_i \cup \{v_{n-i}\})$ is complete and PC, $R_{k',k} \subseteq R_{k',n-i} \circw R_{n-i,k}$. Moreover, because $G(F_i \cup \{v_j\})$ is complete and PC by construction and induction, $R_{k',j} \subseteq R_{k',k} \circw R_{k,j} \subseteq R_{k',n-i} \circw R_{n-i,k} \circw R_{k,j}$.
	
	Thus, after adding a missing edge, the resulting graph remains path-consistent. At last we will get the complete graph, which is equivalent to the completion of $G$. Note that the label of every edge in $G$ is not changed. This finishes the proof. \qed
\end{proof}

\section{Further Discussion}
In this section we discuss the relation of distributive subalgebras with conceptual neighbourhood graphs (CNGs) \cite{Freksa92a} and star distributivity \cite{Montanari} of classical CSPs. 

\subsection{Distributive Subalgebras and Conceptual Neighbourhood Graph}
As we have seen, the classes of convex IA and RCC8 relations are maximal distributive subalgebras of IA and RCC8 respectively. For IA, Ligozat \cite{Ligozat94} characterises the convex relations by using the CNG of IA \cite{Freksa92a} (shown in Table~\ref{tab:int} (ii)). An IA relation is convex if it is an ``interval'' $[\alpha,\beta]$ containing all the relations between its two endpoint relations $\alpha,\beta$ in the CNG. The subclass of convex IA relations is exactly the maximal distributive subalgebra $\mathcal{C}_{\mathrm{IA}}$.

Similar idea applies to PA and RCC5 directly. For PA, the CNG is shown in the left of Figure~\ref{fig:CNG}. From the CNG of PA, we observe the ``convex'' relations correspond to relations in $\mathcal{C}_{\mathrm{PA}} = \{<,=,>,\le,\ge\}$, one of the maximal distributive subalgebras of PA. For RCC5, the CNG is shown in the middle of Figure~\ref{fig:CNG}. The subclass of convex RCC5 relations is precisely the maximal distributive subalgebra  $\mathcal{D}^5_{14}$ specified in \cite{AIJPrime}.

\begin{figure}
\begin{tabular}{ccccc}
\begin{subfigure}[b]{0.18\textwidth}
 \includegraphics[width=\textwidth]{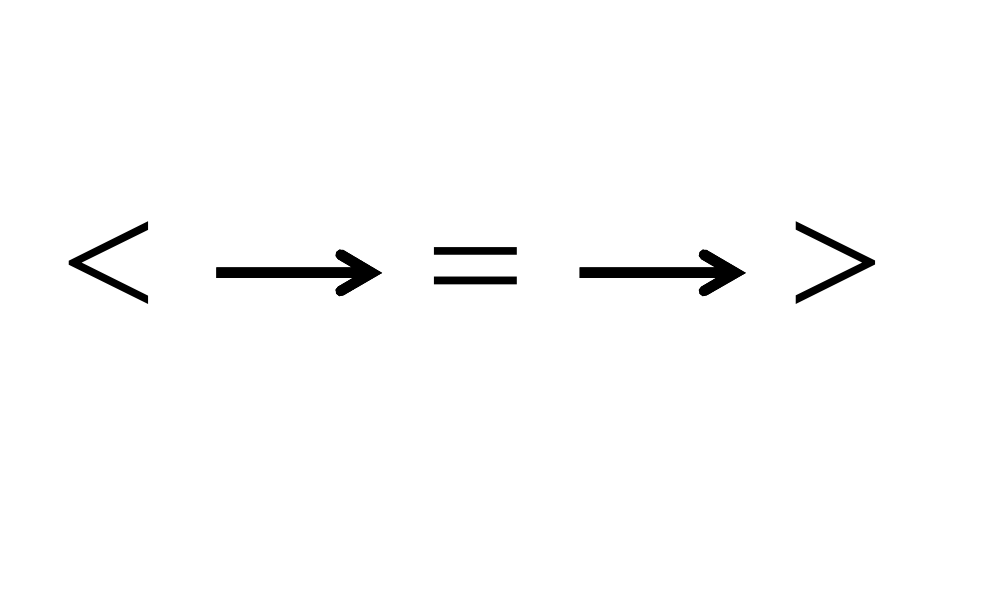}
\end{subfigure}
&\quad &
\begin{subfigure}[b]{0.25\textwidth}
 \includegraphics[width=\textwidth]{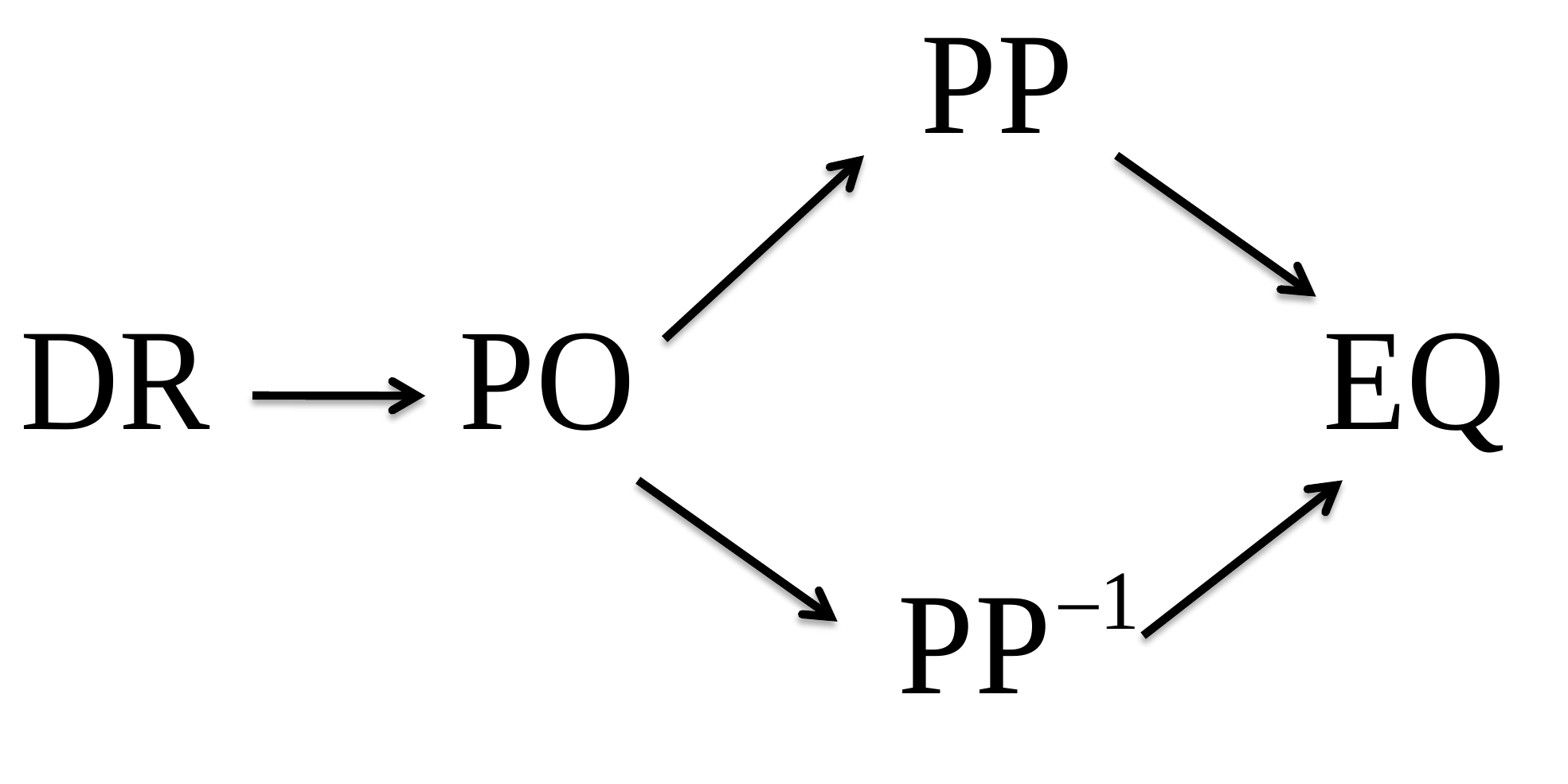}
\end{subfigure}
&\quad &
\begin{subfigure}[b]{0.4\textwidth}
 \includegraphics[width=\textwidth]{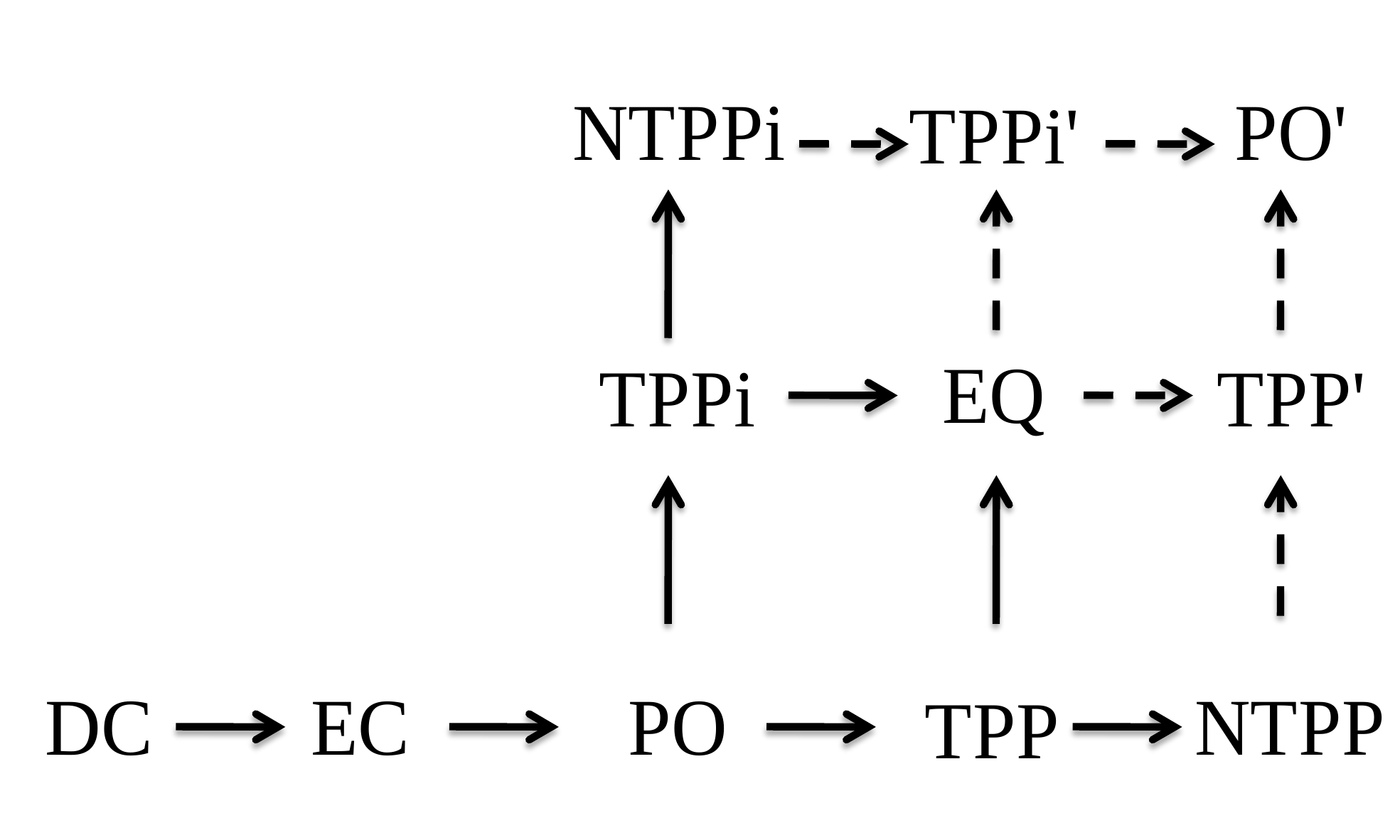}
\end{subfigure}
\end{tabular}
\caption{CNG of PA, RCC5, and RCC8}\label{fig:CNG}
\end{figure}

The CNG of CRA is constructed by using the CNG of PA. For example, note that $<$ and $=$ are conceptual neighbours in the CNG of PA, and $NW$ is defined as $x<x'$ and $y>y'$ and $N$ is defined as $x=x'$ and $y>y'$. Then $N$ and $NW$ should be conceptual neighbours in CRA. The complete CNG of CRA is given in \cite{Ligozat98} and the subclass of convex CRA relations corresponds to the maximal distributive subalgebra that is the Cartesian product of $\mathcal{C}_{\mathrm{PA}}$ and itself.
Like CRA, the CNG of RA is constructed by using the CNG of IA. The subclass of convex RA relations \cite{BalbianiCC99} is the maximal distributive subalgebra that is the Cartesian product of $\mathcal{C}_{\mathrm{IA}}$ and itself.

For RCC8, the situation is a little different. We need to revise the CNG by introducing three imaginary relations ${\btpp}',{\btppi}'$ and ${\bpo}'$ (see Figure~\ref{fig:CNG}, right). After this modification, Chandra and Pujari  \cite{ChandraP05} identified the class of convex RCC8 relations, which is precisely the maximal distributive subalgebra $\mathcal{D}^8_{41}$ specified in \cite{AIJPrime}.

A natural question arises as, ``Can we obtain each maximal distributive subalgebra by designing an appropriate CNG?" The answer seems negative as the maximal distributive subalgebra $\mathcal{S}_{\mathrm{PA}}$ contains $\ne$ but does not contain either $\leq$ or $\geq$. 

\subsection{Relation with Classical CSPs}
For finite domain CSPs, Montanari observed properties similar to the distributivity in this paper. In \cite{Montanari}, Montanari defined two different concept related to distributivity. One is a \emph{distributive set of relations w.r.t. set $X_k$} and the other is \emph{star-distributive constraint network}. 
The second concept is very similar to our notion of distributivity, except that it only requires the relations to form a closure w.r.t. the network. A constraint network over a distributive subalgebra is always star-distributive, but it is not clear whether a star-distributive network is always over a distributive subalgebra. 

As we have seen, relations in a distributive subalgebra exhibit convexity in Helly's sense. In finite CSP, row convex constraints \cite{Beek1995} and (the more general) tree convex constraints \cite{ZhangF08} enjoy a similar property, which is specified w.r.t. the ``rows'' or ``images'' of the constraints rather than the constraints themselves. The relations $R,S,T$ below  are all CRC constraints. Moreover, we have $R\circw (S\cap T) \ne R\circw S \cap R\circw T$ and $R\cap S \ne \varnothing$, $R \cap T \ne \varnothing$, $S\cap T \ne \varnothing$ but $R\cap S\cap T = \varnothing$. This shows that CRC constraints are not always distributive and do not always satisfy Helly's condition \eqref{eq:helly}.


\vspace*{-3mm}

\begin{table*}
  \centering
    \begin{tabular}{cccccc}
     $\left(\begin{array}{ccc}
     1  &  0   &  0\\
     1  &   1  &   0\\
     0   & 0   &  1
  \end{array}\right)$
  &
   $\left(\begin{array}{ccc}
     1   &  1   &  1\\
     0  &   0  &   1\\
     0   & 0   &  1
  \end{array}\right)$
  &
   $\left(\begin{array}{ccc}
     0   &  0   &  1\\
     1  &   1  &   1\\
     0   & 1  &  0
  \end{array}\right)$
  &
  $\left(\begin{array}{ccc}
     0   &  0   &  1\\
     0  &  0  &   1\\
    0  & 0  &  0
  \end{array}\right)$
  & &
   $\left(\begin{array}{ccc}
     0   &  0   &  1\\
     1  &  1  &   1\\
    0 & 0  &  0
  \end{array}\right)$\\
$R$ & $S$ &$T$ & $R\circw (S\cap T)$ & \quad & \quad $ R\circw S \cap R\circw T$
  \end{tabular}
  \end{table*}%




\section{Conclusion}\label{conclusion}
In this paper, we gave a detailed discussion of the important concept of distributive subalgebra proposed in a recent work \cite{AIJPrime}. We proved that distributive subalgebras are exactly subalgebras which are Helly in our sense and found all maximal distributive subalgebras for PA, IA, RCC5/8, CRA, and RA. We also proposed two nice properties of distributive subalgebras which will be used for efficient reasoning of large sparse constraint networks. Future work will implement and empirically evaluate and compare these two methods by using real datasets. 

\section*{Appendix}
Here we give a detailed proof of Theorem~\ref{thm:pc-wgc}. 
\setcounter{theorem}{1}
\begin{theorem}
Let $\qcm$ be a qualitative calculus that satisfies \eqref{eq:ra} and \eqref{eq:pc->c}. Suppose $\mathcal{S}$ is a distributive subalgebra of $\qcm$. Then every path consistent network over $\mathcal{S}$ is weakly globally consistent and minimal.
\end{theorem}
\begin{proof}
 We first note that, since  $\qcm$ satisfies \eqref{eq:ra}, any three relations in $\qcm$ have the cycle law property \eqref{eq:cycle} (see e.g. \cite{duntsch2005relation}), and by Theorem~\ref{thm:dis=helly} any distributive subalgebra $\mathcal{S}$ of $\qcm$ is Helly and hence satisfies \eqref{eq:helly+}.

Suppose $\Gamma = \{v_i R_{ij} v_j: 1\le i,j\le n\}$ is a path consistent network over $\mathcal{S}$. Write $V_{k}=\{v_1,v_2,\ldots,v_k\}$ and $W^{k+1}_{t}=V_t \cup \{v_{k+1}\}$ for $1\le k < n$ and $1\le t \le k$. Let $\Delta_{V_k}=\{v_i \delta_{ij} v_j: v_i,v_j \in V_k\}$ be a consistent scenario of $\Gamma{\downarrow}_{V_k} = \{v_i R_{ij} v_j: v_i,v_j \in V_k\}$ (see Figure~\ref{fig:vk}). 
\begin{figure}
 \centering
 \begin{subfigure}[b]{0.45\textwidth}
 \includegraphics[width=\textwidth]{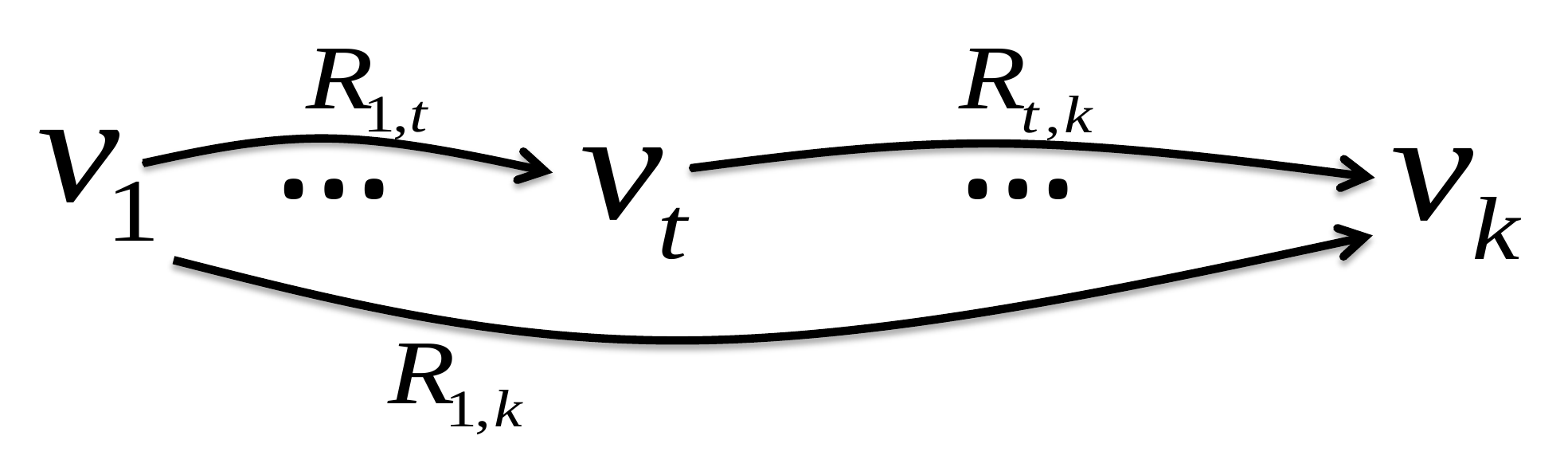}
 \caption{$\Gamma{\downarrow}_{V_k}$}
\end{subfigure}
 \begin{subfigure}[b]{0.45\textwidth}
 \includegraphics[width=\textwidth]{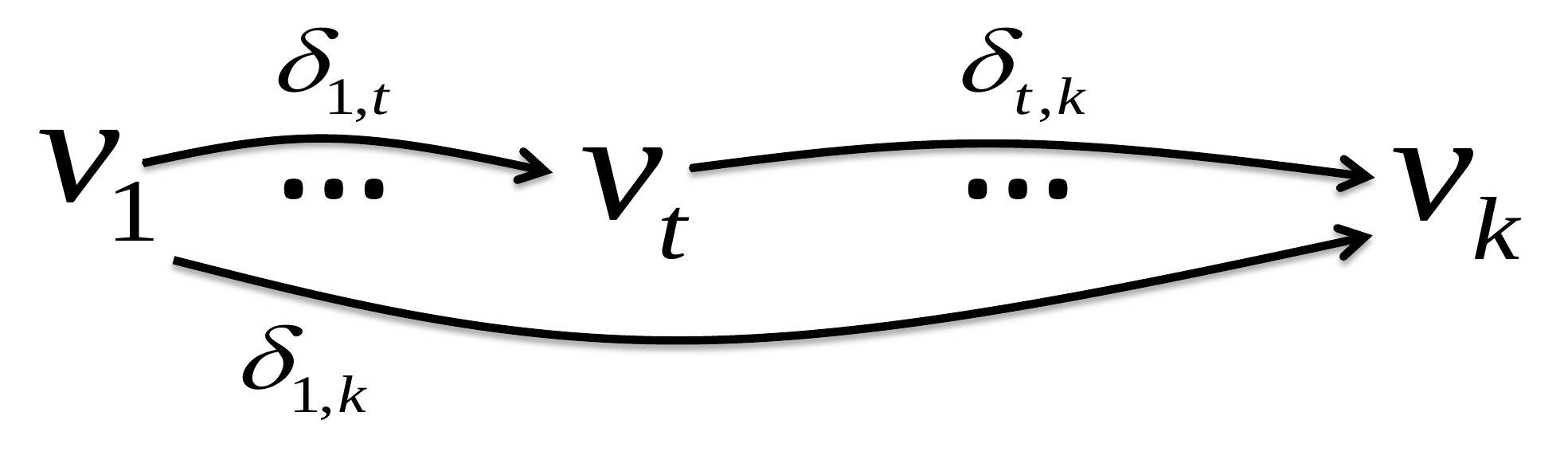}
 \caption{$\Delta_{V_k}$}
\end{subfigure}
 \caption{Illustration of $\Gamma{\downarrow}_{V_k}$ and $\Delta_{V_k}$ in the proof.}\label{fig:vk}
\end{figure}

We show $\Delta_{V_k}$ can be extended to a consistent scenario $\Delta_{V_{k+1}}$ of $\Gamma{\downarrow}_{V_{k+1}}$. 
Note that any path consistent basic network over $\qcm$ is consistent by \eqref{eq:pc->c}.

Let $\widehat{R}_{k+1,i} = \bigcap_{j=1}^k  (R_{k+1,j}\circw \delta_{j,i})$ for $i=1,\ldots,k$. It is easy to see $\widehat{R}_{k+1,i} \subseteq R_{k+1,i}$. Our idea is as follows:
\begin{enumerate}[Step 1.]
 \item Choose an arbitrary basic relation $\delta_{k+1,1}$ in $\widehat{R}_{k+1,1}$. 
 \item Extend a consistent scenario $\Delta_{W^{k+1}_{t}}$ to a consistent scenario $\Delta_{W^{k+1}_{t+1}}$ by choosing a certain basic relation $\delta_{k+1,t+1}$ in $\widehat{R}_{k+1,t+1}$, together with the constraints $\{v_i \delta_{i,t+1} v_{t+1} | 1 \le i \le t\}$ in $\Delta_{V_k}$. See Figure~\ref{fig:extend} for illustration of $\Delta_{W^{k+1}_{t}}$ and $\Delta_{W^{k+1}_{t+1}}$.
 \item Repeat Step 2 for $1\le t \le k-1$ until a consistent scenario $\Delta_{V_{k+1}}$ of $\Gamma{\downarrow}_{V_{k+1}}=\Gamma{\downarrow}_{W^{k+1}_k}$ is obtained.
 \end{enumerate}
 \begin{figure}
  \centering
  \begin{subfigure}[b]{0.45\textwidth}
 \includegraphics[width=\textwidth]{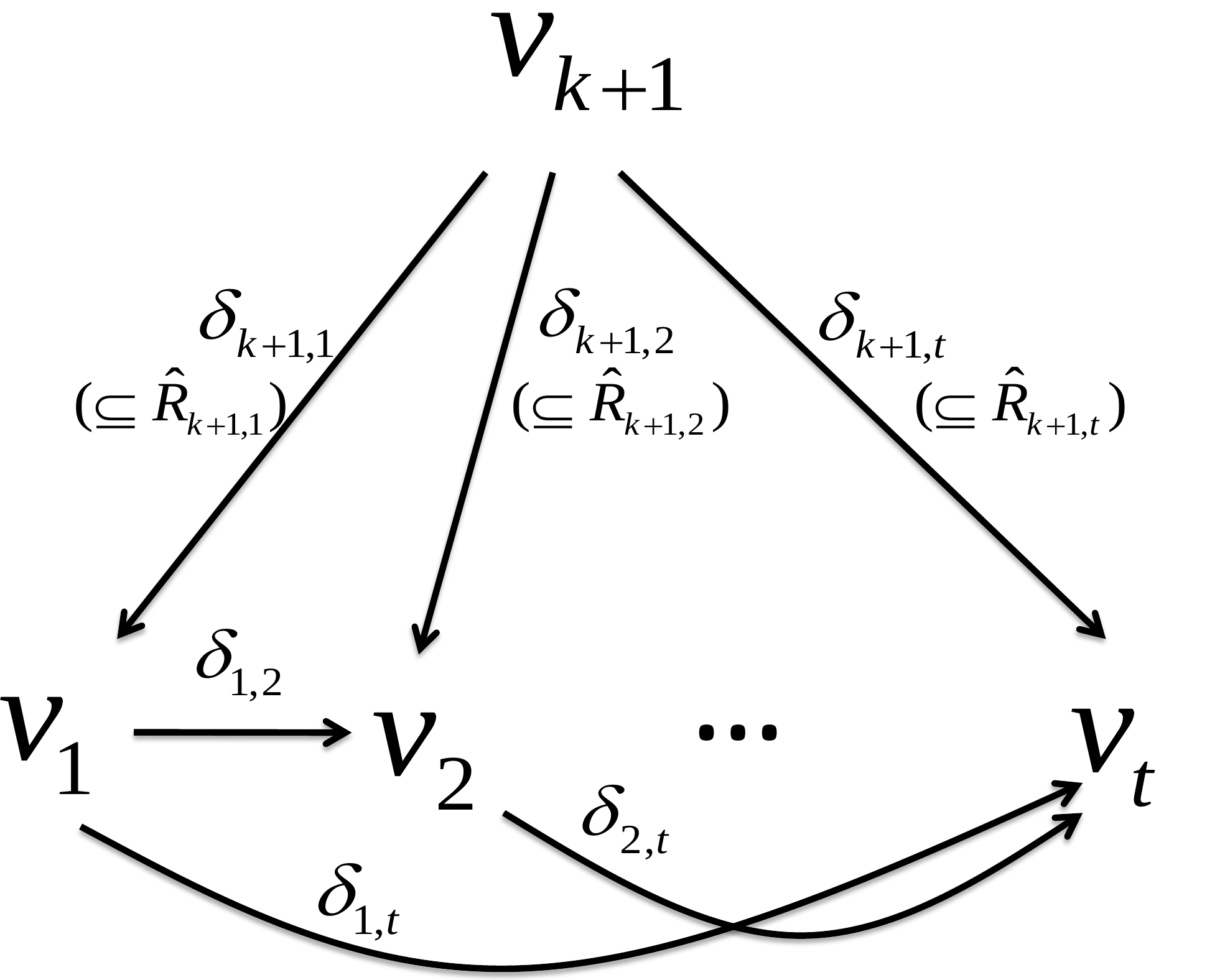}
 \caption{$\Delta_{W^{k+1}_{t}}$}
\end{subfigure}
  \begin{subfigure}[b]{0.45\textwidth}
 \includegraphics[width=\textwidth]{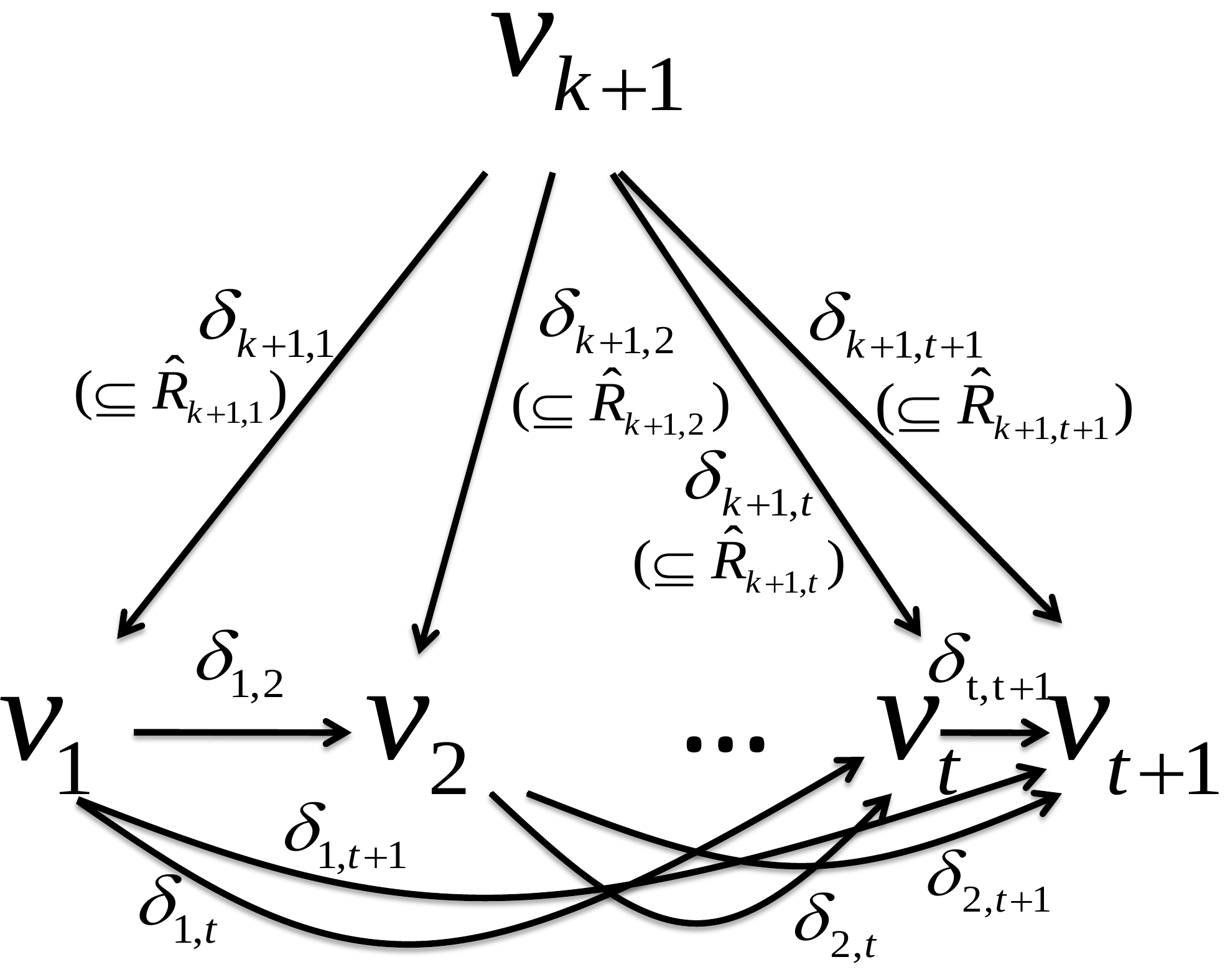}
 \caption{$\Delta_{W^{k+1}_{t+1}}$}
\end{subfigure}
  \caption{Illustration of $\Delta_{W^{k+1}_{t}}$ and $\Delta_{W^{k+1}_{t+1}}$ in the proof.}\label{fig:extend}
 \end{figure}

To show Step 1 can be achieved, we prove that $\widehat{R}_{k+1,i} = \bigcap_{j=1}^k  (R_{k+1,j}\circw \delta_{j,i}) \ne \emptyset$ for all $1\le i \le k$. By applying the cycle law (see \eqref{eq:cycle}), for $1 \le j,j' \le k$  we have
\begin{align*}
(R_{k+1,j}\circw \delta_{ji}) \cap (R_{k+1,j'}\circw \delta_{j'i}) \not=\varnothing &\quad \mbox{iff}\quad  ((R_{k+1,j'}\circw \delta_{j'i})\circw \delta_{ij}) \cap R_{k+1,j}\not=\varnothing\\
&\quad \mbox{iff}\quad (R_{j',k+1}\circw R_{k+1,j}) \cap (\delta_{j'i}\circw \delta_{ij})\not=\varnothing.
\end{align*}
Since $\Gamma$ is path consistent and the partial scenario $\Delta_{V_k}$ is also path consistent, we have $\delta_{j'j} \subseteq R_{j',j} \subseteq R_{j',k+1}\circw R_{k+1,j}$ and $\delta_{j'j} \subseteq \delta_{j'i}\circw \delta_{ij}$. Therefore $(R_{j',k+1}\circw R_{k+1,j}) \cap (\delta_{j'i}\circw \delta_{ij})\not=\varnothing$ and hence $(R_{k+1,j}\circw \delta_{ji}) \cap (R_{k+1,j'}\circw \delta_{j'i}) \not=\varnothing$ for any $1\le j,j'\le k$. Note $\mathcal{S}$ is Helly by Proposition~\ref{thm:dis=helly}, we know $\widehat{R}_{k+1,i} = \bigcap_{j=1}^k  (R_{k+1,j}\circw \delta_{j,i}) \ne \emptyset$ for all $1\le i \le k$. 

To show Step 2 can be achieved, we only need to find a basic relation $\delta_{k+1,t+1}$ in $\widehat{R}_{k+1,t+1}$ such that $\Delta_{W^{k+1}_{t}} \cup \{v_{k+1} \delta_{k+1,t+1} v_{t+1}\}$ is path consistent, for $t=1,\ldots,k-1$. 

With the following statements, we can show the existence of such $\delta_{k+1,t+1}$.
\begin{enumerate}[{Statement} 1.]
  \item $\delta_{k+1,i} \circw \delta_{i,t+1} \cap \widehat{R}_{k+1,t+1} \ne \varnothing$ for any $1\le i \le t$.
 \item  $\delta_{k+1,i} \circw \delta_{i,t+1} \cap \delta_{k+1,j} \circw \delta_{j,t+1} \ne \varnothing$ for any $1\le i,j \le t$.
\end{enumerate}
In fact, from the above statements and that $\mathcal{S}$ is Helly, we know $(\bigcap_{i=1}^t (\delta_{k+1,i} \circw \delta_{i,t+1})) \cap \widehat{R}_{k+1,t+1} \ne \emptyset$. Thus, there exists a $\delta_{k+1,t+1}$ in $\widehat{R}_{k+1,t+1}$ such that
\begin{equation}\label{eq:deltaInt}
 (\bigcap_{i=1}^t (\delta_{k+1,i} \circw \delta_{i,t+1})) \cap \delta_{k+1,t+1} \ne \emptyset. 
\end{equation}
To show this $\delta_{k+1,t+1}$ actually extends $\Delta_{W^{k+1}_{t}}$, we also need to prove that $\Delta_{W^{k+1}_{t}} \cup \{v_i \delta_{i,t+1} v_{t+1} : 1 \le i \le t\} \cup  \{v_{k+1} \delta_{k+1,t+1} v_{t+1}\}$ is path consistent. Note we only need to show $\delta_{k+1,t+1} \subseteq \delta_{k+1,i} \circw \delta_{i,t+1}$ for any $1\le i \le t$, because $\Delta_{W^{k+1}_{t}}$ and $\Delta_{V_{t+1}}(\subseteq \Delta_{V_{k}})$ are both path consistent. This will be true if \eqref{eq:deltaInt} is true. Therefore, in the following, we show the two statements above are actually true.

For Statement 1, note 
\begin{align*}
(\bigcap_{j=1}^k  (R_{k+1,j}\circw \delta_{j,t+1})) \circw \delta_{t+1,i} 
&=\bigcap_{j=1}^k  (R_{k+1,j}\circw \delta_{j,t+1}  \circw \delta_{t+1,i})\\
&\supseteq \bigcap_{j=1}^k (R_{k+1,j} \circw \delta_{ji}).
\end{align*}
Then $\widehat{R}_{k+1,t+1} \circw \delta_{t+1,i} \supseteq \widehat{R}_{k+1,i} \supseteq \delta_{k+1,i} \ne \varnothing$, that is $\delta_{k+1,i} \cap (\widehat{R}_{k+1,t+1} \circw \delta_{t+1,i}) \ne \varnothing$. By cycle law, we have $(\delta_{k+1,i} \circw \delta_{i,t+1}) \cap \widehat{R}_{k+1,t+1} \ne \varnothing$ for any $1\le i \le t$.

For Statement 2, for any $1 \le i,j \le t$ we have
\begin{align*}
(\delta_{k+1,i}\circw \delta_{i,t+1}) \cap (\delta_{k+1,j}\circw \delta_{j,t+1}) \ne \varnothing 
&\quad \mbox{iff}\quad   (\delta_{i,k+1} \circw \delta_{k+1,j} \circw \delta_{i,t+1}) \cap \delta_{j,t+1} \ne \emptyset\\
&\quad \mbox{iff}\quad (\delta_{i,k+1} \circw \delta_{k+1,j}) \cap (\delta_{i,t+1}  \circw \delta_{t+1,j}) \ne \emptyset.
\end{align*}
Because $\Delta_{W^{k+1}_{t}}$ is a (path) consistent scenario, we have $\delta_{ij} \subseteq \delta_{i,k+1} \circw \delta_{k+1,j}$ for $1\le i,j\le t$. Note $\Delta_{V_{t+1}} (\subseteq \Delta_{V_{k}})$ is also a (path) consistent scenario, we have $\delta_{ij} \subseteq \delta_{i,t+1} \circw \delta_{t+1,j}$. Then $(\delta_{i,k+1} \circw \delta_{k+1,j}) \cap (\delta_{i,t+1} \circw \delta_{t+1,j}) \supseteq \delta_{ij} \ne \varnothing$ for $1\le i,j\le t$, and hence $(\delta_{k+1,i}\circw \delta_{i,t+1}) \cap (\delta_{k+1,j}\circw \delta_{j,t+1}) \ne \varnothing$. \qed
\end{proof}


\bibliographystyle{plain}
\bibliography{refdis}
\end{document}